\documentclass[10pt,journal,doublecolumn,compsoc]{IEEEtran}

\ifCLASSOPTIONcompsoc
  \usepackage[nocompress]{cite}
\else
  \usepackage{cite}
\fi
\ifCLASSINFOpdf
\else
\fi
\usepackage{graphicx,semtrans,amsmath,amsfonts,amssymb,bm,hyperref,url,breakurl,epsfig,epsf,color,MnSymbol,mathbbol,fmtcount,algorithmic,algorithm,caption,subcaption,
  multirow,verbatim}

\usepackage{caption}
\usepackage{cite}

\usepackage[bottom,hang,flushmargin]{footmisc}
\usepackage{comment}
\usepackage{cite}
\usepackage{hyperref}

\usepackage{amssymb,amsfonts,amsmath}

\usepackage{lipsum}
\makeatletter
\def\footnoterule{\relax%
  \kern-5pt
  \hbox to \columnwidth{\hfill\vrule width \columnwidth height 0.4pt\hfill}
  \kern4.6pt}
\makeatother

\newcommand{\bi}{\begin{itemize}}
\newcommand{\ei}{\end{itemize}}
\newcommand{\bal}{\begin{align}}
\newcommand{\eal}{\end{align}}

\newcommand{\EE}{\mathbb{E}}

\newcommand{\bx}{\mathbf{x}}

\newcommand{\bp}{\mathbf{p}}
\newcommand{\bw}{\mathbf{w}}
\newcommand{\ba}{\mathbf{a}}
\newcommand{\bv}{\mathbf{v}}
\newcommand{\bc}{\mathbf{c}}

\newcommand{\cX}{\mathcal{X}}

\newcommand{\cN}{\mathcal{N}}

\newcommand{\cG}{\mathcal{G}}

\newcommand{\bb}{\mathbf{b}}
\newcommand{\barb}{\bar{\bb}}
\newcommand{\hatb}{\hat{\bb}}
\newcommand{\cK}{\mathcal{K}}

\def \endprf{\hfill {\vrule height6pt width6pt depth0pt}\medskip}
\newenvironment{proof}{\noindent {\bf Proof} }{\endprf\par}

\newtheorem{theorem}{\textbf{Theorem}}
\newtheorem{lemma}{\textbf{Lemma}}
\newtheorem{corollary}{\textbf{Corollary}}

\newtheorem{definition}{\textbf{Definition}}

\newtheorem{proposition}{\textbf{Proposition}}
\newtheorem{assumption}{\textbf{Assumption}}

\begin{document}
\title{Network Maximal Correlation}

\author{Soheil~Feizi*, Ali~Makhdoumi*, Ken~Duffy, Manolis Kellis, Muriel M{\'e}dard
\IEEEcompsocitemizethanks{\IEEEcompsocthanksitem S. Feizi and A. Makhdoumi contributed equally to this work. S. Feizi is with Stanford Universiy. A. Makhdoumi, M. Kellis and M. M\'edard are with Massachusetts Institute of Technology (MIT). K. Duffy is with Maynooth University. \protect}
\IEEEcompsocitemizethanks{\IEEEcompsocthanksitem Copyright (c) 2016 IEEE. Personal use of this material is permitted. However, permission to use this material for any other purposes must be obtained from the IEEE by sending a request to pubs-permissions@ieee.org. An earlier version of this paper is available on MIT DSpace: \url{https://dspace.mit.edu/handle/1721.1/98878}. \protect}
}

\IEEEtitleabstractindextext{%
\begin{abstract}
We introduce Network Maximal Correlation (NMC) as a multivariate measure of nonlinear association among random variables. NMC is defined via an optimization that infers transformations of variables by maximizing aggregate inner products between transformed variables. For finite discrete and jointly Gaussian random variables, we characterize a solution of the NMC optimization using basis expansion of functions over appropriate basis functions. For finite discrete variables, we propose an algorithm based on alternating conditional expectation to determine NMC. Moreover we propose a distributed algorithm to compute an approximation of NMC for large and dense graphs using graph partitioning. For finite discrete variables, we show that the probability of discrepancy greater than any given level between NMC and NMC computed using empirical distributions decays exponentially fast as the sample size grows. For jointly Gaussian variables, we show that under some conditions the NMC optimization is an instance of the Max-Cut problem. We then illustrate an application of NMC in inference of graphical model for bijective functions of jointly Gaussian variables. Finally, we show NMC's utility in a data application of learning nonlinear dependencies among genes in a cancer dataset.
\end{abstract}

}

\maketitle

\IEEEdisplaynontitleabstractindextext
\IEEEpeerreviewmaketitle
\ifCLASSOPTIONcompsoc
\IEEEraisesectionheading{\section{Introduction}}
\else
\section{Introduction}
\fi

\IEEEPARstart{I}{dentifying} relationships among variables
in large datasets is an increasingly important task in systems
biology \cite{albert2007network}, social sciences
\cite{handcock2002statistical}, finance \cite{haldane2009rethinking},
and other fields. For independent observations of bivariate data,
several measures exist that characterize the strength of the
association based on a linear fit (e.g., Pearson's correlation
\cite{pearson1895note}, canonical correlation \cite{thompson2005canonical,scharf2000canonical}), rank statistics (e.g., Spearman's correlation
\cite{spearman1904proof}), and information content (e.g., mutual
information \cite{cover2012elements,reshef2011detecting}). Some of
these measures have been extended to the multivariate setting. For
instance, Chow and Liu \cite{chow1968approximating} have used mutual
information in the inference of tree graphical models, while Liu et al.
\cite{liu2009nonparanormal,liu2012high} introduced a copula setup based on rank statistics such as Spearman's
\cite{spearman1904proof}
correlation coefficient to characterize graphical models for some
nonlinear functions of jointly Gaussian variables. Another method to capture a nonlinear association between two variables is the randomized dependence coefficient \cite{lopez2013randomized}, where it fixes a set of nonlinear functions and then uses randomized rank statistics to compute association between nonlinear transformations of variables.

A classical measure of nonlinear relationships between two random
elements
$X_1$ and $X_{2}$ is {\it Maximal Correlation} (MC), introduced
by Hirschfeld \cite{hirschfeld1935connection}, Gebelein \cite{gebelein1941statistische}, Sarmanov \cite{sarmanov1962maximum}  and R\'enyi \cite{renyi1959measures}, having also appeared in the work of Witsenhausen \cite{witsenhausen1975sequences}, Ahlswede and G\'acs \cite{ahlswede1976spreading}, and Lancaster \cite{lancaster1957some}.
MC determines possibly nonlinear transformations of two
variables, subject to zero mean and unit variance, to maximize
their Pearson's correlation. MC not only computes an association
strength between variables, but it also characterizes a possible functional
relationship between them.
\begin{definition}[Maximal Correlation]\label{def:pmc}
\textup{
Let $(\Omega,\mathcal{F},P)$ be a probability triple and for
$k\in\{1,2\}$ let $(\mathcal{X}_k,\beta_k)$ be a measurable space with
$X_k:\Omega\to\mathcal{X}_k$ a random element. Maximal Correlation (MC) between the two (not necessarily real-valued) random elements $X_1$
and $X_2$ is defined as
\begin{align}\label{eq:pmc}
\rho(X_1, X_2)\triangleq
	\sup_{\phi_1, \phi_{2}}\ \EE[\phi_1(X_1)\ \phi_{2}(X_{2})],
\end{align}
such that $\phi_k: \mathcal{X}_k \to \mathbb{R}$ is Borel measurable,
$\EE[\phi_k(X_k)]=0$, and $\EE[\phi_k(X_k)^2]=1$, for $k=1,2$.}
\end{definition}
MC can be computed efficiently for both discrete
\cite{breiman1985estimating} and continuous real-valued
\cite{lancaster1957some} random variables (Section
\ref{sec:mc}). For discrete random variables, under mild
conditions, MC is equal to the second largest singular value of an
scaled joint probability distribution matrix and the optimal transformations of the variables can be
characterized using right and left singular vectors of the scaled
probability distribution matrix. Alternating Conditional Expectation (ACE) was introduced by Breiman and Friedman \cite{breiman1985estimating} to compute MC, and was further analyzed by Buja \cite{buja1990remarks}. Recently, MC has been used in
many applications. It is related to strong data processing inequalities and contraction coefficients, being recently investigated by Anantharam et al. \cite{anantharam2013maximal}, Polyanskiy \cite{polyanskiy2013hypothesis}, and Raginsky \cite{raginsky2014strong}. Further studies include applications in information theory \cite{
courtade2013outer, calmon2014exploration}, information-theoretic
security and privacy \cite{li2015maximal,calmon2014information, calmon2013bounds, calmon2015fundamental, makhdoumi2015forgot},
data processing \cite{raginsky2014strong, anantharam2013maximal}
and other fields \cite{beigi2015duality, etesami2015maximal,
farnia2015minimum, razaviyayn2015discrete, makur2015bounds}. In particular, Beigi and Gohari \cite{beigi2015duality} has introduced multipartite maximal correlation and showed its connection with the multipartite hypercontractivity ribbon. Multipartite maximal correlation aims to find a correlation matrix $C$ with the largest $r$ such that $C- r I$ is positive semidefinite, where $I$ is the identity matrix. This objective function is different than the one of the network maximal correlation \eqref{eq:nmc} discussed in the present paper.

Many modern datasets consist of independent observations of
high-dimensional multivariate random elements $X_1,\ldots,X_n$
(i.e., $n\gg 2$). One approach to characterize relationships among
these elements is to determine the bivariate MC for each pair
of them. In this approach one would solve the following optimization
for all pairs $(i,i')$, $1\leq i,i'\leq n$:
\begin{align}\label{eq:multiple-MC}
\sup_{\phi_{i,i'},\phi_{i',i}}\ \EE\left[\phi_{i,i'}(X_{i})\ \phi_{i',i}(X_{i'}) \right],
\end{align}
such that $\phi_{i,i'}: \mathcal{X}_i \to \mathbb{R}$ is Borel measurable, $\EE[\phi_{i,i'}(X_i)]=0$, and $\EE[\phi_{i,i'}(X_i)^2]=1$, for $1\leq i,i'\leq n$.

Should they exist, let $\phi_{i,i'}^*$ for $1\leq i,i'\leq n$ denote the resulting
optimizers. By this approach, each element
$X_i$ is assigned to $n-1$ transformation functions $\{\phi_{i,i'}^*:1
\leq i'\leq n, ~ i' \neq i \}$. In some applications, there may be interpretability
and over-fitting issues. To circumvent these issues, we consider
an alternate multivariate extension of the MC optimization
\eqref{eq:pmc} with two types of constraints. We seek to formulate
an optimization that (i) assigns a single transformation function
to each element; and (ii) its objective function can be restricted to a
subset of variable pairs. The latter conditioning is described by a graph and is
motivated by several distinct considerations: (a) there may be {\it
a priori} known structure that indicates some elements are necessarily
unrelated; (b) one may not care about the association of certain
elements; or (c) the restriction may serve as a computation reduction
technique.
\begin{definition}[Network Maximal Correlation]\label{def:nmc}
\textup{
Let $G=(V,E)$ be a graph with vertices $V=\{1,2,\ldots, n\}$ and edges
$E \subseteq \{(i,i'):i,i'\in V, i\neq i'\}$. The
Network Maximal Correlation (NMC) of $X_1,
\dots, X_n$ given $G$ is defined to be
\begin{align}\label{eq:nmc}
\rho_{G}(X_1,\ldots,X_n)\triangleq \sup_{\phi_1, \dots, \phi_n} \ \sum_{(i, i')\in E} \mathbb{E} \left[\phi_i(X_i)\ \phi_{i'}(X_{i'}) \right],
\end{align}
such that $\phi_i: \mathcal{X}_i \to \mathbb{R}$ is Borel measurable,
$\EE \left[ \phi_i(X_i) \right]=0$, and $\EE \left[ \phi_i(X_i)^2 \right]=1$, for all $1\leq
i\leq n$.
}
\end{definition}
When no confusion arises, we use $\rho_{G}$ to refer to the NMC. The NMC optimization maximizes the aggregate inner products between transformed variables. NMC naturally generalizes the bivariate MC to the multivariate setting. For example, the bivariate MC is equivalent to the NMC when the graph $G=(V,E)$ has two nodes connected by an edge. That is, $V=\{1,2\}$ and $E=\{(1,2)\}$. As another example, if $V=\{1,2,3\}$ and $E=\{(1,2),(1,3)\}$, the objective function of the NMC optimization \eqref{eq:nmc} aims to find transformation functions $\left\{\phi_i(X_i)\right\}_{i=1}^{3}$ which maximize \[\mathbb{E} \left[ \phi_1(X_1)\ \phi_{2}(X_{2})+\phi_1(X_1)\ \phi_{3}(X_{3}) \right].\] Note that in this example the transformation function $\phi_1(X_1)$ appears in two terms in the objective function, leading to additional coupling constraints in the optimization \eqref{eq:nmc}, when compared to multiple bivariate MC optimizations \eqref{eq:multiple-MC}. We investigate this optimization for a general graph $G=(V,E)$. In applications, we highlight appropriate selections of $G$ in the NMC optimization \footnote{We use the terminology {\it graph} and {\it network} interchangeably.}.

Since $\mathbb{E}[\phi_i(X_i)^2]=1$ for any $1 \le i\le n$, we have
 \begin{align}\label{eq:nmc-mse}
\rho_G= |E|- \inf_{\phi_1, \dots, \phi_n}\ \frac{1}{2} \sum_{(i, i')\in E} \mathbb{E}[(\phi_i(X_i)-\phi_{i'}(X_{i'}))^2],
\end{align}
which means that the NMC optimization \eqref{eq:nmc} is equivalent to finding functions of random variables that minimize the Mean Squared Error (MSE) among all neighboring random variables. The form \eqref{eq:nmc-mse} can be useful in different applications such as fitting a nonlinear regression model (e.g., $\phi_i(X_i)=\sum_{i'=\{1,...,n\} \setminus \{i\}} \phi_{i'}(X_i')+Z$).

The techniques described here to characterize local and global optima of the aforementioned NMC optimization can be used in other related formulations as well. For example, we also introduce absolute NMC, which maximizes the total absolute pairwise correlations among variables (Section \ref{Sec:SI-NMC}). Absolute NMC is appropriate for applications where the strength of an association does not depend on the sign of the correlation coefficient. Moreover, the NMC optimization can be regularized to consider fewer nonlinear transformations or to restrict the set of possible transformations, further avoiding over-fitting issues (Section \ref{Sec:SI-NMC}).

In Section \ref{sec:nmc}, for finite discrete random variables, we characterize the solution of the NMC optimization using a natural basis expansion. For finite discrete random variables, we show that the NMC optimization is an instance of the Maximum Correlation Problem (MCP), which is NP-hard \cite{chu1993multivariate, hotelling1935most, hotelling1936relations, zhang2012computing}.

In Section \ref{sec:computation}, using the results from the Multivariate Eigenvector Problem (MEP) \cite{chu1993multivariate}, we characterize necessary conditions that are satisfied at the global optimum of the NMC optimization. We propose an efficient algorithm based on Alternating Conditional Expectation (ACE) \cite{breiman1985estimating} that converges to a local optimum of the NMC optimization. We also provide guidelines for choosing appropriate starting points of the ACE algorithm to avoid local maximizers. The proposed algorithm does not require the formation of joint distribution matrices which could be expensive for variables with large number of categories. We develop a distributed version of the ACE algorithm to compute an approximation of the NMC value for large dense graphs using graph partitioning. We characterize a bound on the expected performance of the proposed algorithm for poly-growth graphs which are often used in analyzing distributed graph algorithms (see e.g., \cite{klein1993excluded, rao1999small, karger2002finding}).

In Section \ref{subsec:nmc-robustness}, we prove a finite sample bound and error guarantees for NMC. Under some conditions we prove that NMC of finite discrete random variables is continuous with respect to their joint probability distributions. That is, a small perturbation in the joint distribution results in a small change in the NMC value. Moreover, we show that the probability of discrepancy greater than any given level between NMC and NMC computed using empirical distributions decays exponentially fast as the sample size grows.

In Section \ref{sec:nmc-gaussian}, for jointly Gaussian variables, we use Hermite-Chebyshev polynomials as the basis for the functions of Gaussian random variables and characterize the solution of the NMC optimization. Under some conditions, we show that the NMC optimization is equivalent to the Max-Cut problem, which is NP-complete \cite{garey1976some}. However, there exist algorithms to approximate its solution using Semidefinte Programming (SDP) within certain approximation factors \cite{goemans1995improved}. Moreover, we provide sufficient conditions under which a solution of the NMC optimization can be characterized exactly.

In Section \ref{sec:applications}, we illustrate some applications of NMC. First we show an application of NMC in inference of graphical model (Definition \ref{def:graphica-model}) for bijective functions of jointly Gaussian variables. Graphical models provide a useful framework for characterizing dependencies among variables and for efficient computation of marginal and the mode of distributions \cite{wainwright2008graphical}. Moreover, Gaussian graphical models play an important role in applications such as linear regression \cite{neter1996applied}, partial correlation \cite{baba2004partial}, maximum likelihood estimation \cite{banerjee2008model}, etc. Here we consider a setup in which observed variables are related to latent jointly Gaussian variables through link functions. These link functions are unknown, bijective, and can be linear or nonlinear. We show that, under some conditions, the inverse covariance matrix of the transformed variables computed by the NMC optimization (as well as multiple MC optimizations) characterizes the underlying graphical model. With an example we show that when underlying link functions are monotone the performance of our inference method is comparable to the one of the {\it copula} model developed by Liu et al. \cite{liu2009nonparanormal,liu2012high}. Moreover, we illustrate that if we violate the model assumption for both our inference framework and for the copula method by considering non-monotone link functions, our inference method appears to outperform the copula one, highlighting the robustness of our proposed inference framework. Finally, we provide an example to demonstrate NMC's utility in a data application where we apply sample NMC to cancer datasets \cite{li2014potential} and infer nonlinear gene association networks. Over these inferred networks, we deduce gene modules that are significantly associated with survival times of individuals while these modules are not detected using linear association measures.

\section{Network Maximal Correlation}\label{sec:nmc}
\subsection{Review of Maximal Correlation}\label{sec:mc}
Recall the definition of maximal correlation (Definition \ref{def:pmc}). For $k=i, i'$, we let $\phi_k^*(X_k)$ denote a solution of optimization \eqref{eq:pmc} should it exist. The existence and uniqueness of a solution $(\phi_i^*(X_i), \phi_{i' }^*(X_{i'}))$  to the MC optimization \eqref{eq:pmc}  have been investigated in \cite{breiman1985estimating}. Maximal correlation $\rho(X_i, X_{i'})$ is always between 0 and 1, where a high MC value indicates a strong association between two variables \cite{gebelein1941statistische}.

Unlike Pearson's linear correlation \cite{pearson1895note}, MC only depends on the joint distribution of the variables and not on their sample spaces $\cX_i$. Several works have investigated different aspects of optimization \eqref{eq:pmc} for both discrete \cite{breiman1985estimating} and continuous \cite{lancaster1957some} random variables. For discrete variables, under mild conditions, MC is equal to the second largest singular value of the following $Q$-matrix \cite{breiman1985estimating}:
\begin{definition}\label{def:Q-matrix}
Let $P_{X_i,X_{i'}}$ be the joint distribution of discrete variables $X_{i}$ and $X_{i'}$ with finite alphabet sizes $|\cX_i|$ and $|\cX_{i'}|$, respectively. Define a matrix $Q_{i,i'}\in \mathbb{R}^{|\cX_i|\times |\cX_{i'}|}$ whose $(j,j')$ element is
\begin{align}\label{eq:Q-matrix}
Q_{i,i'}(j,j')\triangleq \frac{P_{X_i, X_{i'}}(j ,j')}{\sqrt{P_{X_i}(j) P_{X_{i'}}(j')}}.
\end{align}
The matrix is called the $Q$-matrix of the distribution $P_{X_i,X_{i'}}$.
\end{definition}
In this case, optimal transformations of variables can be characterized using right and left singular vectors of the normalized probability distribution matrix.

For Gaussian variables, Lancaster \cite{lancaster1957some} introduced a basis expansion with Hermite-Chebyshev polynomials to compute MC. Interestingly in this case the maximal correlation and Pearson's linear correlation is equivalent.

\subsection{NMC Formulation}
Recall the definition of the Network Maximal Correlation (Definition \ref{def:nmc}). NMC infers nonlinear transformation functions assigned to each node variable so that the aggregate pairwise correlations over the graph $G=(V,E)$ is maximized.

Let $\phi_i^*(\cdot)$ be a solution of the NMC optimization \eqref{eq:nmc} should it exist. Then, an edge maximal correlation of variables $i$ and $i'$ is defined as
\begin{align*}
\rho_{G}^{(i,i')}(X_1,\cdots,X_n)\triangleq \big|\mathbb{E}[\phi_i^*(X_i)\ \phi_{i'}^*(X_{i'})]\big|,
\end{align*}
where $(i,i')\in E$. Unlike the bivariate MC formulation of \eqref{eq:pmc}, the edge maximal correlation is a function of the joint distribution of all variables. Thus, an edge maximal correlation of variables $X_i$ and $X_{i'}$ is always smaller than or equal to their bivariate maximal correlation, i.e.,
\begin{align*}
\rho_{G}^{(i,i')}(X_1,\cdots,X_n)\leq \rho(X_i,X_{i'}).\nonumber
\end{align*}
We next provide a framework to study NMC and its properties. Following the definition provided in \cite[Section 3]{witsenhausen1975sequences}, for $i=1, 2, \dots, n$, we define a set of real-valued measurable functions $H_{X_i}$ as
\begin{align}\label{def:Hi}
H_{X_i}=& \{ \phi_i \mid \phi_i: \mathcal{X}_i \to \mathbb{R}, \phi_i \text{ is Borel measurable},\\
& \EE[\phi_i \circ X_i]<\infty, ~ \EE[(\phi_i \circ X_i)^2]<\infty \},\nonumber
\end{align}
where the inner product of two elements of $H_{X_i}$ is defined as $\EE[\phi_i(X_i)\ \phi'_i(X_i)]$, and the norm is defined as $\sqrt{\EE[(\phi_i \circ X_i)^2]}$. Note that $\phi_i(X_i)$ is a random variable in $\mathbb{R}$ such that for any Borel-measurable set $B$, we have $\mathbb{P}[\phi_i(X_i) \in B]= \mathbb{P}[X_i \in \phi_i^{-1}(B)]$. We use the notation $\phi_i \circ X_i$ and $\phi_i(X_i)$ interchangeably.

We let $ \{ \psi_{i,j} \}_{j=1}^{\infty}$ represent an orthonormal basis for $H_{X_i}$. We will explicitly construct such basis in the case of discrete random variables and Gaussian random variables, studied in this paper.
\begin{theorem}\label{thm:hilbert-nmc}
Consider the following optimization:
\begin{align}\label{eq:nmc-opt-withbasis}
\sup_{\{\{a_{i,j}\}_{i=1}^{n}\}_{j=1}^{\infty}} \quad & \sum_{(i,i')\in E}\ \sum_{j,j'} a_{i,j} a_{i',j'}\ \rho_{i,i'}^{j,j'} \\
&\sum_{j=1}^{\infty} a_{i,j} ^2 =1, ~~ 1\leq i \leq n, \nonumber \\
& \sum_{j=1}^{\infty} a_{i,j}\ \mathbb{E}[\psi_{i,j} (X_i)]=0, ~~ 1\leq i \leq n,\nonumber
\end{align}
where
\begin{align}\label{eq:rho-def}
\rho_{i,i'}^{j,j'}\triangleq \mathbb{E}[\psi_{i,j}(X_i)\ \psi_{i',j'}(X_{i'})].
\end{align}
Let $\phi_i^*(\cdot)$ and $\{a_{i,j}^*\}_{j=1}^{\infty}$ for $1\leq i\leq n$ be solutions of optimizations \eqref{eq:nmc} and \eqref{eq:nmc-opt-withbasis}, respectively. Then,
\begin{align}\label{eq:nmc-optimum-phi-coefs}
\phi_i^*(X_i) = \sum_{j=1}^{\infty} a_{i,j}^*\ \psi_{i,j} (X_i).
\end{align}
is a solution of the NMC optimization \eqref{eq:nmc}. Moreover, $\rho_{i,i'}^{j,j'}$ for $j,j'\geq 1$ are coefficients of the basis expansion of $P_{X_i, X_{i'}}$ with respect to the basis $\left\{ \psi_{i,j} \psi_{i',j'} \right\}_{j,j'}$, i.e.,
\begin{align*}
P_{X_i, X_{i'}} (x_i, x_{i'})= \sum_{j, j'} \rho_{i,i'}^{j,j'} \psi_{i,j} (x_i) \psi_{i',j'} (x_{i'}).
\end{align*}
\end{theorem}
\begin{proof}
A proof is presented in Section \ref{subsec:proof-hilbert-nmc}.
\end{proof}

Equation \eqref{eq:nmc-optimum-phi-coefs} means convergence in the $L_2$ sense, i.e., $\lim_{n\to \infty} ||\phi_i^*(X_i)-\sum_{j=1}^{n} a_{i,j}^*\ \psi_{i,j} (X_i)|| =0$. Throughout the paper since we only work with the inner product of functions, without loss of generality, we replace $\phi_i^*(X_i)$ by $\sum_{j=1}^{\infty} a_{i,j}^*\ \psi_{i,j} (X_i)$.  Also note that for the case of finite discrete random variables, the convergence is in fact pointwise. For further details, see \cite{lancaster1957some, bryc2005maximum}.

Selecting appropriate set of functions $H_{X_i}$ is critical to have a tractable optimization problem \eqref{eq:nmc-opt-withbasis}. In the following, we consider the NMC optimization for finite discrete variables, while the Gaussian case is discussed in Section \ref{sec:nmc-gaussian}.

\section{NMC for Finite Discrete Variables}\label{sec:computation}
In this section, we analyze the NMC optimization \eqref{eq:nmc} for finite discrete random variables, and then introduce an efficient algorithm to compute NMC. We then propose a parallelizable version of the NMC algorithm based on network partitioning and characterize a bound for its expected performance.

First we introduce some notation. For any vector $\mathbf{v}= (v_1, \dots, v_d) \in \mathbb{R}^d$ and $p \ge 1$, we let $\|\bv\|_p$ represent the standard $p$-norm of the vector $\bv$ defined as
\begin{align*}
||\mathbf{v}||_p= \left( \sum_{i=1}^d v_i^p\right)^{\frac{1}{p}}.
\end{align*}
For $p=2$, we drop the subscript, i.e., $||\mathbf{v}||=||\mathbf{v}||_2$. The infinite norm of a vector is defined as
\begin{align*}
||\mathbf{v}||_{\infty}= \max_{1 \leq i\leq d} v_i.
\end{align*}
The inner product between two vectors $\bv$ and $\bw$ is defined as
\begin{align*}
<\bv,\bw>=\sum_{i=1}^{d} v_i w_i.
\end{align*}
For a matrix $V\in\mathbb{R}^{d_1}\times \mathbb{R}^{d_1}$, the matrix norm is a vector norm on $\mathbb{R}^{d_1\times d_2}$. I.e.,
\begin{align*}
||V||_p=\sup_{\bw\neq 0} \frac{||V \bw||_p}{||\bw||_p}.
\end{align*}
For $p=2$, we drop the subscript.
\vspace{-0.1cm}
\subsection{Relationship Between NMC and Maximum Correlation Problem}\label{ex:NMCdiscrete}
Let $X_i$ be a discrete random variable with alphabet $\{1, \dots , |\mathcal{X}_i| \}$. Let
\begin{align}\label{eq:basis-discrete}
\psi_{i,j}(x)=\mathbf{1}\{x=j\} \frac{1}{\sqrt{P_{X_i}(j)}}
\end{align}
be an orthonormal basis for $H_{X_i}$, where $\mathbf{1}$ is the indicator function. We assume that all the elements of the alphabet $x_i\in\cX_i$ have positive probabilities, as otherwise they can be neglected without loss of generality. We can write
\begin{align*}
\rho_{i,i'}^{j,j'} = \mathbb{E}[\psi_{i,j}(X_i)\ \psi_{i',j'}(X_{i'})] = \frac{P_{X_i,X_{i'}}(j , j')}{\sqrt{P_{X_i}(j) \ P_{X_{i'}}(j')}}.\nonumber
\end{align*}
Therefore, the optimization \eqref{eq:nmc-opt-withbasis} is simplified to the following:
\begin{align}\label{eq:nmc-opt-dis1}
\max_{\{\{a_{i,j}\}_{i=1}^{n}\}_{j=1}^{\infty}} \quad & \sum_{(i,i')\in E}\  \sum_{j,j'} a_{i,j} a_{i',j'}\ \frac{P_{X_i,X_{i'}}(j , j')}{\sqrt{P_{X_i}(j) \ P_{X_{i'}}(j')}} \\
& \sum_{j=1}^{|\mathcal{X}_i|} (a_{i,j})^2 =1, ~~ 1\leq i \leq n, \nonumber \\
& \sum_{j=1}^{|\mathcal{X}_i|} a_{i,j}\ \sqrt{P_{X_i}(j)}=0, ~~ 1\leq i \leq n.\nonumber
\end{align}
For $i=1, \dots, n$, let
\begin{align}\label{def:ai-sqrtpi}
&\ba_i\triangleq \left( a_{i,1},a_{i,2},\ldots,a_{i,|\mathcal{X}_{i}| } \right) ^T \\
&\sqrt{\bp_i}\triangleq \left( \sqrt{P_{X_i}(1)},\sqrt{P_{X_i}(2)},\ldots,\sqrt{P_{X_i}(|\mathcal{X}_{i}| )}\right) ^T.\nonumber
\end{align}
Therefore, the optimization \eqref{eq:nmc-opt-dis1} can be re-written as follows:
\begin{align}\label{eq:nmc-opt-dis2}
\max_{\{\ba_i\}_{i=1}^{n}} \quad & \sum_{(i,i')\in E}  \ba_i^T Q_{i,i'} \ba_{i'} \\
& \|\ba_i\|_{2}=1, ~~ 1\leq i \leq n, \nonumber \\
& \ba_i \perp \sqrt{\bp_i}, ~~ 1\leq i \leq n, \nonumber
\end{align}
where $Q_{i,i'}$ is defined in Definition \ref{def:Q-matrix} and $\perp$ represents orthogonality between two vectors.

The optimization \eqref{eq:nmc-opt-dis2} is not convex nor concave in general. Below, we show that this optimization is an instance of the Maximum Correlation Problem (MCP) proposed by Hotelling \cite{hotelling1935most, hotelling1936relations}. By making this connection, we use established techniques via the Multivariate Eigenvalue Problem (MEP) to solve optimization \eqref{eq:nmc-opt-dis2}.

For each $i$, since $ I_{|\mathcal{X}_i|} - \sqrt{\mathbf{p}_i} \sqrt{\mathbf{p}_i}^T $ is positive semidefinite, we take its square root\footnote{The square root of a symmetric positive semidefinite matrix $A$ is defined as $\sqrt{A}= U \Sigma^{1/2} U^T$ where $A=U \Sigma U^T$.} and write
\begin{align}\label{eq:Bi}
B_i \triangleq \sqrt{I_{|\mathcal{X}_i|}- \sqrt{\mathbf{p}_i} \sqrt{\mathbf{p}_i}^T},
\end{align}
where $I_{|\mathcal{X}_i|}$ is a $|\mathcal{X}_i| \times |\mathcal{X}_i|$ identity matrix. Let $\mathbf{b}_i = B_i \mathbf{a}_i$. Let $U_i \Sigma_i U_i^T $ be the singular value decomposition of $B_i$ where $U^{(j)}_i$ is the $j$-th column of $U_i$ and $\sigma_i^{(j)}$ is the $j$-th singular value of $B_i$. We will show that only one singular value of $B_i$ is zero which corresponds to the singular vector $\sqrt{\bp_i}$. Without loss of generality, suppose $\sigma_i^{(1)}=0$, for all $i$. Define $A_i$, a $|\mathcal{X}_i| \times |\mathcal{X}_i|$ matrix, as follows:
\begin{align}\label{eq:Ai-inverseBi}
A_i\triangleq \Bigg( &\left[U^{(2)}_i, \dots, U^{(|\mathcal{X}_i|)}_i \right] \text{diag}\left( \frac{1}{\sigma_i^{(2)}}, \dots, \frac{1}{\sigma_i^{(|\mathcal{X}_i|)}} \right)\\
&\left[ U^{(2)}_i, \dots, U^{(|\mathcal{X}_i|)}_i \right]^T\Bigg).\nonumber
\end{align}
As we show in the proof of Theorem \ref{thm:NMC2MCP}, $\sigma_{i}^{(j)}\neq 0$, for all $1\leq i\leq n$, and $j\geq 2$, $A_i$ is well-defined according to \eqref{eq:Ai-inverseBi}.
\begin{theorem}\label{thm:NMC2MCP}
The NMC optimization \eqref{eq:nmc-opt-dis2} can be re-written as follows:
\begin{align}\label{eq:nmc2mcp-opt}
\max_{\mathbf{b}_1, \dots, \mathbf{b}_n} \quad &\sum_{(i, i') \in E} \mathbf{b}_i^T A_i^T \left( Q_{i,i'}- \sqrt{\mathbf{p}_i} \sqrt{\mathbf{p}_{i'}}^T \right)  A_{i'} \mathbf{b}_{i'} \\
& \text{s.t. } ||\mathbf{b}_i||_2=1 \qquad 1 \le i \le n.\nonumber
\end{align}
\end{theorem}
\begin{proof}
A proof is presented in Section \ref{proof:thm:nmc2mcp}.
\end{proof}
Let $C$ be a matrix consisting of submatrices $C_{i,i'} \in \mathbb{R}^{|\mathcal{X}_i| \times |\mathcal{X}_{i'}|}$ where if $(i,i')\in E$,
\begin{align}
C_{i,i'}\triangleq A_i^T \left( Q_{i,i'}- \sqrt{\mathbf{p}_i} \sqrt{\mathbf{p}_{i'}}^T \right)  A_{i'},
\end{align}
otherwise $C_{i,i'}$ is an all zero matrix of size $|\mathcal{X}_i| \times |\mathcal{X}_{i'}|$. Note that since the graph $G=(V,E)$ does not have self-loops, $C_{i,i}$ is a zero matrix for $1\leq i\leq n$.

Let $\mathbf{b} \triangleq (\mathbf{b}^T_1, \dots, \mathbf{b}^T_n)^T \in \mathbb{R}^{|\cX|}$, where $|\cX|=\sum_{i=1}^n |\mathcal{X}_i|$ and $\mathbf{b}_i \in \mathbb{R}^{|\mathcal{X}_i|}$.
\begin{corollary}\label{prop:nmc2mcp-standard}
The NMC optimization \eqref{eq:nmc2mcp-opt} can be written as follows:
\begin{align}\label{eq:opt-standard-mcp}
\max_{\mathbf{b}} \quad &\mathbf{b}^T C \mathbf{b}\\
& ||\mathbf{b}_i||_2=1, \quad 1 \leq i\leq n.\nonumber
\end{align}
\end{corollary}
The optimization \eqref{eq:opt-standard-mcp} is in the standard form of the MCP problem, proposed by
Hotelling \cite{hotelling1935most, hotelling1936relations}, to find the linear combination of one set of variables that correlates maximally with the linear combination of another set of variables.
\begin{definition}[Multivariate Eigenvalue Problem \cite{chu1993multivariate} ]\label{def:MEP}
\textup{The first-order optimality condition for optimization \eqref{eq:opt-standard-mcp} is the existence of real-valued scalars, namely, Lagrange multipliers $\lambda_1, \dots, \lambda_n$, such that the following system of equations is satisfied:
\begin{align}\label{eq:MEPformulation}
& \sum_{i'=1}^n C_{i,i'} \mathbf{b}_{i'}= \lambda_i \mathbf{b}_i, ~~~ 1 \le i \le n \nonumber \\
& ||\mathbf{b}_i||_2=1, ~~~ 1 \le i \le n.
\end{align}
This system of equations is called Multivariate Eigenvalue Problem (MEP) (\cite{chu1993multivariate, zhang2012computing})}.
\end{definition}
\vspace{-0.1cm}
\subsection{An ACE Approach to Compute NMC}\label{sec:nmc-comp}
In Section \ref{ex:NMCdiscrete}, we established a connection among the NMC optimization \eqref{eq:nmc}, the Maximum Correlation Problem (MCP) and the Multivariate Eigenvalue Problem (MEP) (see e.g., \cite{chu1993multivariate, hotelling1935most, hotelling1936relations, zhang2012computing}). After showing that the NMC optimization can be reformulated as the MCP, we use existing techniques from the literature to solve it. Several local maximizers exist for cases where finding a global optimum of optimization \eqref{eq:opt-standard-mcp} is computationally difficult \cite{horst1965factor, chu1993multivariate}. For example, an aggregated power method that iterates on blocks of $C$ was proposed by Horst \cite{horst1961relations} as a general technique for solving the MEP (Definition \ref{def:MEP}) numerically.

Below, we summarize general algorithmic ideas to solve MCP:
\begin{itemize}
\item [(1)] First, an efficient algorithm is used to solve MEP which is the necessary first order condition for MCP. This step is studied in references \cite{chu1993multivariate, horst1961relations}.
\item [(2)] Next, a strategy is used to choose starting points of the algorithm or to jump out of the local minima of optimization \eqref{eq:opt-standard-mcp}. We adopt this step from \cite{zhang2012computing, zhang2012alternating}.
\end{itemize}
\textbf{An efficient algorithm to solve MEP:} Algorithm \ref{alg:MEP} proposed by \cite{chu1993multivariate} is a Gauss-Seidel algorithm \cite{bertsekas1999nonlinear} to solve MEP, which is a variant of the classical power iteration method (see e.g. \cite{golub2012matrix}).
\begin{algorithm}[t]
\caption{Gauss-Seidel Algorithm for MEP}
\begin{algorithmic}\label{alg:MEP}
\STATE \textbf{Input:}  $C \in \mathbb{R}^{|\cX|} \times \mathbb{R}^{|\cX|}$.
\STATE \textbf{Initialization:}  $\mathbf{b}^{(0)} \in \mathbb{R}^{|\cX|}$.
\STATE \textbf{for} {$k=0, 1, \dots$}
\STATE \hspace{.2 in} \textbf{for} {$i=1, \dots, n$}
~~~\STATE \hspace{.4 in}  $\tilde{\mathbf{b}}_i^{(k)}= \sum_{i'=1}^{i-1} C_{i,i'} \mathbf{b}_{i'}^{(k+1)}+ \sum_{i'=i+1}^n C_{i,i'} \mathbf{b}_{i'}^{(k)}$.
~~~\STATE \hspace{.4 in} $\mathbf{b}_i^{(k+1)}= \frac{\tilde{\mathbf{b}}_i^{(k)}}{||\tilde{\mathbf{b}}_i^{(k)}||_2}$
~~~\STATE \hspace{.2 in} \textbf{end}
\STATE \textbf{end}
\end{algorithmic}
\end{algorithm}
Let
\begin{align}\label{eq:r-lam-def}
r(\mathbf{b}) &= \mathbf{b}^T C \mathbf{b}=\sum_{i,i'=1}^{n} \bb_i^T C_{i,i'} \bb_{i'} \\
\lambda_i(\mathbf{b})& =\mathbf{b}_i^T [C_{i,1}, \dots, C_{i,n}] \mathbf{b}=\sum_{i'=1}^{n} \bb_i^T C_{i,i'} \bb_{i'} \nonumber\\
\end{align}
Algorithm \ref{alg:MEP} is an iterative algorithm. Let $\bb^{(k)}=(\bb_1^{(k)},\ldots,\bb_n^{(k)})$ be the update of Algorithm \ref{alg:MEP} at iteration $k$.
\begin{theorem} [\!\! Theorem 5.1 \cite{zhang2012computing}] \label{thm:MEPMCP}
The sequence $\{r(\mathbf{b}^{(k)}) \}$ generated by Algorithm \autoref{alg:MEP} is monotonically increasing and convergent.
\end{theorem}
Theorem \ref{thm:MEPMCP} indicates that the Algorithm \ref{alg:MEP} finds a local optimum of the optimization \eqref{eq:opt-standard-mcp}. According to Proposition \ref{prop:nmc2mcp-standard}, this solution provides a local optimum for the NMC optimization \eqref{eq:nmc}. In the following we introduce a strategy that uses a local optimum of the optimization \eqref{eq:opt-standard-mcp} and constructs another solution for the optimization \eqref{eq:opt-standard-mcp} with strictly higher objective function value.
\\\textbf{A strategy for avoiding local optima of MCP:}
\begin{proposition}\label{prop:jump-local}
Let $\barb$ be a solution of the MEP \eqref{eq:MEPformulation}. Suppose that there exists an $1 \le i \le n$ such that $\lambda_i(\barb) < 0 $. Let $\hatb=(\hatb_1,\ldots,\hatb_n)$ be defined as: $\hatb_i=-\barb_i$, for any $i$ such that $\lambda_i(\barb) < 0$, and $\hatb_{i'}=\barb_{i'}$, otherwise. Then, we have $r(\hatb)>r(\barb)$.
\end{proposition}
\begin{proof}
A proof is presented in Section \ref{proof-prop-jump-local}.
\end{proof}
Algorithm \autoref{alg:MEP} finds a local optimum $\barb$ for the optimization \eqref{eq:opt-standard-mcp}. This solution can be translated to a solution of the NMC optimization \eqref{eq:nmc} according to Theorem \ref{thm:hilbert-nmc}, equations \eqref{eq:basis-discrete}, \eqref{eq:Bi}, and using $\mathbf{b}_i = B_i \mathbf{a}_i$. This leads to a direct algorithm to find a solution for the NMC optimization \eqref{eq:nmc} based on alternating conditional expectation (Algorithm \ref{alg:nmc}). Similarly to Algorithm \ref{alg:MEP}, Algorithm \ref{alg:nmc} converges to the local optimum of the NMC optimization \eqref{eq:nmc}. We use a strategy similarly to the one of Proposition \ref{prop:jump-local} to avoid remaining at local maximizers. At each iteration of the algorithm, we update transformation functions as follows: Suppose at iteration $k$, transformation functions are $\{\phi_i^{(k)}\}_{i=1}^n$. If we fix all variables except the transformation function of node $i$, an optimizer of $\phi_i^{(k+1)}$ can be written as the normalized conditional expectation of functions of its neighbors. In each update, the objective function of the NMC optimization increases or remains the same. Note that, for the bivariate case (i.e., $n=2$), Algorithm \ref{alg:nmc} is simplified to the ACE algorithm \cite{breiman1985estimating} for the MC computation.
\begin{algorithm}[t]
\caption{Network ACE to compute NMC}
\begin{algorithmic}\label{alg:nmc}
\STATE \textbf{Input:} $G$, $X_1, \dots, X_n$,
\STATE \textbf{Initialization:}  $\phi_1^{(0)}(X_1), \dots , \phi_n^{(0)}(X_n)$ with mean zero and unit variance.
\STATE \textbf{for} $k=0,1, \dots$
\STATE \hspace{.2 in} \textbf{for} $i=1,..., n$
\STATE \hspace{.4 in} $\tilde{\phi_i}^{(k)}(X_i)=\EE \left[\sum_{i'=1}^{i-1} \phi_{i'}^{(k+1)}(X_{i'})|X_{i} \right]+\EE \left[\sum_{i'=i+1}^{n} \phi_{i'}^{(k)}(X_{i'})|X_{i} \right]$, for $i'\in \cN(i)$.
\STATE \hspace{.4 in} \textbf{update:} $\phi_i^{(k+1)}(X_i)= \frac{\tilde{\phi_i}^{(k)}(X_i)}{\sqrt{\EE \left[{\tilde{\phi_i}^{(k)}(X_i)}^2 \right]}}$
\STATE \hspace{.2 in} \textbf{end}
\STATE \hspace{.2 in} $\rho^{{(k+1)}}_G=\sum_{(i, i')\in E} \mathbb{E} \left[ \phi_i^{(k+1)}(X_i)\phi_{i'}^{(k+1)}(X_{i'}) \right]$
\STATE \textbf{end}
\end{algorithmic}
\end{algorithm}
If the number of nodes $n$ is large, then computation of NMC may be expensive. In the following, we propose an approach to compute NMC using parallel computation.
\vspace{-0.1cm}
\subsection{Parallel Computation of NMC}\label{subsec:nmc-comp-parallel}
For large and dense networks, exact computation of NMC may become computationally expensive. For those cases, we propose a parallelizable algorithm which approximates NMC using network partitioning. The idea can be described as follow. For a given graph $G=(V,E)$,
\begin{itemize}
\item [(1)] Partition the graph into small disjoint sets.
\item [(2)] Estimate NMC for each partition independently.
\item [(3)] Combine NMC solutions over sub-graphs to form an approximation of NMC for the original graph.
\end{itemize}
A $k$-partition $\pi(k,V,E)$ of graph $G=(V,E)$ is defined as a set $\{V_1,\dots,V_M\}$ such that, for any $1 \leq i\neq j\leq M$, $V_i\cap V_j=\emptyset$, $\bigcup_{i=1}^{M}=V$, and $1 \leq |V_i|\leq k$. We say an edge $e\in E$ belongs to $\pi(k,V,E)$ if $e\in \bigcup_{i=1}^{M} V_i\times V_i$.
\begin{definition}\label{def:k-partition}
\textup{
A $k$-partitioning of graph ${G}=(V, E)$ is a set $\Pi(k,V,E)=\{\pi_1(k,V,E),\pi_2(k,V,E),...\}$ where each $\pi_i(k,V,E)$ is a $k$-partition of the graph.
}
\end{definition}
Next we define an $(\epsilon, k)$-partitioning of graph ${G}=(V, E)$.
\begin{definition}[\cite{jung2008local}, Section A]\label{def:partition}
\textup{
An $(\epsilon, k)$-partitioning of graph ${G}=(V, E)$ is a distribution over all $k$-partitioning of the graph such that, for any $e \in E$, $\EE[\mathbf{1}\{e\notin \pi_i(k,V,E)\}] \le \epsilon$.}
\end{definition}
\begin{definition}
\textup{
A graph ${G}$ is poly-growth if there exists $r>0$ and $C>0$, such that for any vertex $v$ in the graph,
\begin{align*}
|N_v(d)| \le C d^r,
\end{align*}
where $N_v(d)$ is the number of nodes within distance $d$ of $v$ in $G$ (distance is defined as the shortest path length on the graph).
}
\end{definition}
Reference \cite{jung2008local} describes the following procedure for generating an $(\epsilon, k)-$ partitioning on a graph:
\begin{itemize}
\item [1.] Given $G=(V, E)$, $k$, and $\epsilon > 0$, we define the truncated geometric distribution as follows:
\begin{align}\label{eq:trunc-geo-dist}
\mathbb{P}[x=l] = \left\{
  \begin{array}{l l}
    \epsilon (1- \epsilon)^{l-1}, & \quad l < k,\\
    (1- \epsilon)^{k-1}, & \quad l=k.
  \end{array} \right.
\end{align}
\item [2.] We then order nodes arbitrarily $1, \dots, N$. For node $i$, we sample $R_i$ according to distribution \eqref{eq:trunc-geo-dist} and assign all nodes within that distance from node $i$ to color $i$. If a node is already colored, we re-color it with color $i$.
\item [3.] All nodes with the same color form a $k$-partition of the graph.
\end{itemize}
\begin{proposition}[\cite{jung2008local}, Lemma 2]
If $G$ is a poly-growth graph, then by selecting $k = \Theta (\frac{r}{\epsilon} \log \frac{r}{\epsilon})$, the above procedure results in an $(\epsilon, C k^r)$-partitioning.
\end{proposition}
Next, we use an $(\epsilon,k)$-graph partitioning to approximate NMC over large graphs using parallel computations. Consider the following approach:
\begin{itemize}
\item [(1)] We sample a partition $\{V_1, \dots, V_M\}$ of $V$, given an $(\epsilon, k)$- partitioning of ${G}$.
\item [(2)] For each partition $1 \le m \le M$, we compute NMC restricted to ${G}_m=(V_m, E\cap (V_m \times V_m))$, denoted by $\hat{\rho}_{{G}_m}$.
\item [(3)] Let $\hat{\rho}_{{G}}= \sum_{m=1}^M \hat{\rho}_{{G}_m}$ be an approximation of $\rho_{{G}}$.
\end{itemize}
In the following, we bound the approximation error by bounding boundary effects:
\begin{theorem}\label{th:approxnmc}
Consider an $(\epsilon, k)$- partitioning of the graph ${G}$. We have,
\begin{align} \label{eq:approxnmc}
  \mathbb{E}[\hat{\rho}_{{G}}] \ge (1- \epsilon) \rho_{{G}},
\end{align}
where the expectation is over $(\epsilon, k)$- partitioning of graph $G$.
\end{theorem}
\begin{proof}
A proof is presented in Section \ref{proof:th:approxnmc}.
\end{proof}
\section{Properties of Network Maximal Correlation for Finite Discrete Random Variables}\label{subsec:nmc-robustness}
In many applications, often only noisy measurements or samples of joint distributions are observed. In this section, we prove a finite sample generalization bound, and error guarantees for Network Maximal Correlation of finite discrete random variables. Specifically, under general conditions, we prove that Network Maximal Correlation is a continuous measure with respect to the joint probability distribution. That is, a small perturbation in the distribution results in a small change in the NMC value. Moreover we prove that the probability of discrepancy between NMC and sample NMC (NMC computed using empirical distributions) greater than any given level decays exponentially fast as the sample size $m$ grows.

Throughout this subsection we only consider finite discrete random variables. Moreover to simplify notation we let $P$ be the matrix representation of probability distribution $P_{X_1,\cdots,X_n}$. We assume that all the elements of the alphabet $x_i\in\cX_i$ have positive probabilities, as otherwise they can be neglected without loss of generality, and define
\begin{align}\label{eq:delta-i}
\delta_{X_i}(P)\triangleq\arg\min_{x_i\in\cX_i} P_{X_i}(x_i)>0,\quad  1\leq i\leq n.
\end{align}
The empirical distribution of these variables using $m$ observed samples is defined as $P^{(m)} (x_1, \dots, x_n)=\frac{1}{m}\sum_{s=1}^m \mathbf{1}\{x_1^{(s)}=x_1, \dots, x_n^{(s)}=x_n\}$, where $\{(x_1^{(s)}, \dots, x_n^{(s)})\}_{s=1}^m$ are i.i.d. samples drawn according to the distribution $P_{X_1, \dots, X_n}$. The vector of observed samples of variable $X_i$ is denoted by $\bx_i=(x_i^{(1)},x_i^{(2)},\ldots,x_i^{(m)})$.
\vspace{-0.1cm}
\subsection{Continuity of Network Maximal Correlation}
Let $P$ and $\tilde{P}$ be two distributions over alphabets $({\cX_1}, \dots, {\cX_{n}})$. Let $K=\max_{ 1 \le i \le n}|\mathcal{X}_i|$. Thus, $\cK\triangleq K^n$ is an upper bound on the alphabet size of the joint distribution. In the following, we show that if the infinity norm distance between $P$ and $\tilde{P}$ is small (i.e., $||P-\tilde{P}||_{\infty}= \max_{x_1 \in \mathcal{X}_1, \dots, x_n \in \mathcal{X}_n } | P(x_1, \dots, x_n)- \tilde{P}(x_1, \dots, x_n)| \le \gamma$), their corresponding NMC values (denoted by $\rho_{G}$ and $\tilde{\rho}_{G}$ respectively) are close to each other.
\begin{theorem}\label{thm:continiousnmc}
Network Maximal Correlation is a continuous function of the joint probability distribution $P$. Let $||P-\tilde{P}||_{\infty} \le \gamma$, for $\gamma \le \delta^{3/2} \mathcal{K}^{-1}$. Then, we have
\begin{align}\label{eq:thm:contNMC}
|\rho_G-\tilde{\rho}_G|  \le \gamma \cK |E|  \frac{8}{\delta^2},
\end{align}
where $\delta=\min_{1 \le i \le n} \left( \min \{\delta_{X_i}(P), \delta_{X_i}(\tilde{P})\} \right)$.
\end{theorem}
\begin{proof}
A proof is presented in Section \ref{proof:thm:continiousnmc}.
\end{proof}
\begin{corollary}
Let $\rho$ and $\tilde{\rho}$ be bivariate MCs with respect to distributions $P$ and $\tilde{P}$, respectively. Let $||P-\tilde{P}||_{\infty} \le \gamma$. For any $\gamma \le \delta^{3/2} \mathcal{K}^{-1}$, we have
\begin{align}
|\rho-\tilde{\rho}|  \le \gamma \cK \frac{8}{\delta^2},
\end{align}
where $\cK$ and $\delta$ are defined according to Theorem \ref{thm:continiousnmc}.
\end{corollary}
\vspace{-0.1cm}
\subsection{Sample NMC}
Let $\{( x_1^{(s)}, \dots, x_{n}^{(s)})\}_{s=1}^m$ be i.i.d. samples drawn according to a distribution $P_{X_1, \dots, X_n}$. Let $P^{(m)}$ denote the empirical distribution obtained from these samples. Network Maximal Correlation computed using this empirical probability distribution is called {\it Sample Network Maximal Correlation} and is denoted by $\rho_G^{(m)}(\{( x_1^{(s)}, \dots, x_{n}^{(s)})\}_{s=1}^m)$. For simplicity, when no confusion arises, we refer to the sample NMC by $\rho_G^{(m)}$. In the following, we show that the probability of discrepancy greater than any given value between $\rho_G^{(m)}$ and $\rho_{G}$ decays exponentially fast as the sample size $m$ grows.
\begin{theorem}\label{thm:samplenmc}
For any $\eta, \epsilon > 0$, if
\begin{align}\label{eq:thm:samplenmc}
m \ge 2 \left(\frac{32 |E| \cK}{\epsilon \delta'^2} \right)^2 \log \left( \frac{8 n }{\eta} \right),
\end{align}
then we have
\begin{align}
\mathbb{P}[|\rho_G^{(m)}- \rho_G| > \epsilon ] \le \eta,
\end{align}
where $\delta'=\min_{1 \le i \le n} \delta_{X_i}(P) $.
\end{theorem}
\begin{proof}
A proof is presented in Section \ref{proof:thm:samplenmc}.
\end{proof}
Note that the number of required samples to learn the joint probability distribution reliably is a function of the alphabet size. This is reflected in the right hand side of the bound provided in Theorem \ref{thm:samplenmc} through the term $\cK$.
\begin{corollary}
Let $\rho$ and $\rho^{(m)}$ be bivariate MC and sample MC, respectively. For
\begin{align}
m \ge 2 \left(\frac{32  \cK}{\epsilon \delta'^2} \right)^2 \log \left( \frac{16 }{\eta} \right).
\end{align}
we have
\begin{align}
\mathbb{P}[|\rho^{(m)}- \rho| > \epsilon ] \le \eta,
\end{align}
where $\cK$ and $\delta'$ are defined according to Theorem \ref{thm:samplenmc}.
\end{corollary}
\vspace{-0.1cm}
\section{NMC for Jointly Gaussian Variables}\label{sec:nmc-gaussian}
Suppose that $(X_1, \dots, X_n)$ are jointly Gaussian variables with zero means and unit variances. Let $\rho_{i,i'}$ be the correlation coefficient of variables $X_i$ and $X_{i'}$. We assume that $|\rho_{i,i'}|\neq 1$ if $i\neq i'$. Let $G_c=(V_c,E_c)$ be the covariance graph corresponding to these variables where $V_c=\{1,2,...,n\}$, and $(i,i')\in E_c$ iff $\rho_{i,i'}\neq 0$. Moreover for continuous variables $\{\phi_i(.)\}_{i=1}^n$ in optimization \eqref{eq:nmc} are assumed to be continuous and $l_2$ (equivalent to $L_2$ for discrete random variables).

The $k$-th Hermite-Chebyshev polynomial \cite{lancaster1957some} is defined as
\begin{align}\label{eq:hermitte-poly}
\Psi_k(x)\triangleq(-1)^k e^{x^2}\frac{d^k}{dx^k}e^{-x^2}.
\end{align}
These polynomials form an orthonormal basis with respect to Gaussian distributions \cite{lancaster1957some}. That is,
\begin{align}\label{eq:orth-hermitte}
 \int_{-\infty}^\infty \Psi_j(x_i) \Psi_{j'}(x_{i'}) p(x_i, x_{i'}) dx_i dx_{i'} = (\rho_{i,i'})^j \mathbf{1}\{j=j'\},
\end{align}
where $p(x_i, x_{i'})$ is the joint density function of Gaussian variables $X_i$ and $X_{i'}$ with correlation $\rho_{i,i'}$. Let $\psi_{i,j}$ to be the $j$-th Hermite-Chebyshev polynomial, for $1\leq i\leq n$. We have
\begin{align}\label{eq:rho-gaussian}
\rho_{i,i'}^{j,j'} &= \mathbb{E}[\psi_{i,j}(X_i)\ \psi_{i',j'}(X_{i'})]\\
&= (\rho_{i,i'})^{j} \mathbf{1}\{j=j'\}.\nonumber
\end{align}
Moreover, using the definition of Hermite-Chebyshev polynomials \eqref{eq:hermitte-poly}, we have
\begin{align}\label{eq:orth-gaussian}
\mathbb{E}[\psi_{i,j}(X_i)]= \mathbf{1}\{j=0\}, ~~ 1\leq i\leq n.
\end{align}
because all of these functions for $j\geq 1$ have zero means when integrated against a Gaussian distribution. Therefore, using \eqref{eq:rho-gaussian} and \eqref{eq:orth-gaussian}, the optimization \eqref{eq:nmc-opt-withbasis} can be written as
\begin{align}\label{eq:nmc-gauss1}
\max_{\{\{a_{i,j}\}_{i=1}^{n}\}_{j=1}^{\infty}} \quad & \sum_{(i,i')\in E} \ \sum_{j=1}^{\infty} a_{i,j} a_{i',j}\ (\rho_{i,i'})^j\\
& \sum_{j=1}^{\infty} (a_{i,j})^2 =1, ~~ 1\leq i \leq n. \nonumber
\end{align}
We establish that solving the optimization \eqref{eq:nmc-gauss1} is NP-complete by identifying one instance of this problem that reduces to the max-cut problem, which is NP-complete \cite{garey1976some}.
\begin{theorem}\label{thm-nmc-guass-gen}
Let $s_i\in \{-1,1\}$ for $1\leq i\leq n$. Suppose
\begin{align*}
& \sum_{i'\neq i} (1-s_i s_{i'})\rho_{i,i'} \geq 0,\quad \forall 1\leq i\leq n,
\end{align*}
and
\begin{align*}
& \sum_{i'\neq i} s_i s_{i'} \rho_{i,i'} \geq \sum_{i'\neq i} (\rho_{i,i'})^2,\quad \forall 1\leq i\leq n,
\end{align*}
then, $\ba_{i}^*=(0,s_i,0,\dots,0)$, for $1\leq i\leq n$ is a global maximizer of the NMC optimization \eqref{eq:nmc-gauss1} over a complete graph without self-loops.
\end{theorem}
\begin{proof}
\textup{A proof is presented in Section \ref{proof:thm-nmc-guass-gen}.}
\end{proof}
\begin{proposition}\label{prop:nmc-max-cut}
Under assumptions of Theorem \ref{thm-nmc-guass-gen}, the NMC optimization \eqref{eq:nmc} is simplified to the following max-cut optimization:
\begin{align}\label{eq:nmc-max-cut}
\max_{s_{i}} \quad & \sum_{i\neq i'}  s_i s_{i'}\ \rho_{i,i'}\\
& s_{i}\in \{-1,1\}, ~~1\leq i \leq n. \nonumber
\end{align}
Moreover, for all $1\leq i\leq n$, we have $\phi_i^*(X_i)=s_i^* X_i$, where $\phi_i^*$ and $s_i^*$ are solutions of optimizations \eqref{alg:nmc} and \eqref{eq:nmc-max-cut}, respectively.
\end{proposition}
\begin{proof}
\textup{A proof is presented in Section \ref{proof:nmc-max-cut}.}
\end{proof}

For bivariate jointly Gaussian variables, the conditions of Proposition \ref{prop:nmc-max-cut} are always satisfied. For multivariate jointly Gaussian variables however an optimal NMC solution $\phi_i^*(X_i)$ can be different than $\pm X_i$. Proposition \ref{prop:nmc-max-cut} provides conditions under which $\phi_i^*(X_i)=\pm X_i$ for the multivariate setup.

In general, the Max-Cut optimization \eqref{eq:nmc-max-cut} is NP-complete \cite{garey1976some}. However, there exist algorithms to approximate its solution using semidefinte programming (SDP) \cite{goemans1995improved}.
\begin{assumption}\label{asump1}
\textup{Let $(X_1,...,X_n)$ be jointly Gaussian variables, with zero means and unit variances such that $\rho_{i,i'}\neq 1$ if $i\neq i'$, and
\begin{align}
\sum_{i'\neq i} \rho_{i,i'} \geq \sum_{i'\neq i} (\rho_{i,i'})^2,\quad \forall 1\leq i\leq n,
\end{align}
where $\rho_{i,i'}$ is the correlation coefficient of $X_i$ and $X_i'$.}
\end{assumption}
If all correlation coefficients are non-negative (i.e., $\rho_{i,i'}\geq 0$ for $1\leq i,i'\leq n$), Assumption \ref{asump1} is immediately satisfied. However, Assumption \ref{asump1} is more general as some correlation coefficients can be negative and the condition will still hold.
\begin{corollary}\label{cor:nmc-guass-all-ones}
Under Assumption \ref{asump1}, $\phi_i^*(X_i)=X_i$ is a solution of the NMC optimization \eqref{eq:nmc-gauss1}.
\end{corollary}

\section{Illustration of NMC's Use}\label{sec:applications}
In this section, we illustrate some applications of NMC.

\subsection{Inference of Graphical Models for Functions of Gaussian Variables}\label{subsec:inference-guassian-nets}
Suppose that $(X_1, \dots, X_n)$ are jointly Gaussian variables with the covariance matrix $\Lambda_X$. Without loss of generality, we assume all variables have zero means and unit variances. I.e., $\EE[X_i]=0$ and $\EE[X_i^2]=1$, for all $1\leq i\leq n$.

\begin{definition}\label{def:graphica-model}
The graphical model $G_X=(V_X,E_X)$ is defined such that if $(i,i')\notin E_X$, then
\begin{align}
X_i \perp\!\!\!\perp X_{i'} | \{X_{k},k\neq i,i'\},
\end{align}
where $\perp\!\!\!\perp$ represents independence between variables.
\end{definition}
Let $J_X$ be the information (precision) matrix \cite{wainwright2008graphical} of these variables where $J_X=\Lambda_X^{-1}$.
\begin{theorem}[\!\! \cite{wainwright2008graphical}, Example 3.3]\label{thm:guassian-graphical-model}
For jointly Gaussian variables, $(i,i')\in E_X$ if and only if $J_X(i,i')\neq 0$.
\end{theorem}
This Theorem has been stated in other references as well (e.g. \cite{whittaker2009graphical}). Theorem \ref{thm:guassian-graphical-model} represents a way to model explicitly the joint distribution of Gaussian variables using a graphical model $G_X=(V_X,E_X)$. This result is critical in several applications involved with Gaussian variables which requires computation of marginal distributions, or computation of the mode of the distribution. These computations can be performed efficiently over the graphical model using belief propagation approaches \cite{wainwright2008graphical}. Moreover, Gaussian graphical models play an important role in many applications such as linear regression \cite{neter1996applied}, partial correlation \cite{baba2004partial}, maximum likelihood estimation \cite{banerjee2008model}, etc. In many applications, even if variables are not jointly Gaussian, a Gaussian approximation is used often, partially owing to the efficient inference of their graphical models.

In the following, under some conditions, we use the multiple MC \eqref{eq:multiple-MC} and NMC \eqref{eq:nmc} optimizations to characterize graphical models for functions of latent jointly Gaussian variables. These functions are unknown, bijective, and can be linear or nonlinear. More precisely, let $Y_i=f_i(X_i)$, where $f_i:\mathbb{R}\to\mathbb{R}$ is a bijective function. Our goal is to characterize a graphical model for the variables $(Y_1,Y_2,\ldots,Y_n)$ without knowledge of the functions $f_i(\cdot)$.

Consider the following optimizations:
\begin{align}\label{eq:multiple-MC-Y}
\sup_{g_{i,i'},g_{i',i}}\ \EE[g_{i,i'}(Y_{i})\ g_{i',i}(Y_{i'})],
\end{align}
such that $g_{i,i'}: \mathcal{Y}_i \to \mathbb{R}$ is Borel measurable, $\EE[g_{i,i'}(Y_i)]=0$, and $\EE[g_{i,i'}(Y_i)^2]=1$, for $1\leq i,i'\leq n$, and

\begin{align}\label{eq:nmc-Y}
\max_{g_i:\mathbb{R}\to\mathbb{R}} \quad & \sum_{(i, i')} \mathbb{E}[g_i(Y_i)\ g_{i'}(Y_{i'})],\\
&\mathbb{E}[g_i(Y_i)]=0, ~~ 1\leq i \leq n,\nonumber\\
&\mathbb{E}[g_i^2(Y_i)]=1, ~~ 1\leq i \leq n,\nonumber
\end{align}
such that $g_{i}: \mathcal{Y}_i \to \mathbb{R}$ is Borel measurable, $\EE[g_{i}(Y_i)]=0$, and $\EE[g_{i}(Y_i)^2]=1$, for $1\leq i\leq n$. Optimization \eqref{eq:multiple-MC-Y} solves multiple MC between all pairs of variables while optimization \eqref{eq:nmc-Y} solves NMC considering a fully connected graph. Suppose $g_{i,i'}^*$ and $g_i^*$ represent solutions of \eqref{eq:multiple-MC-Y} and \eqref{eq:nmc-Y}, respectively.
\begin{theorem}\label{thm:nonlinear-guassian-graphical-model}
$g_{i,i'}^{*}(Y_i)=\pm X_i$ for all $1 \leq i,i'\leq n$ is a solution of \eqref{eq:multiple-MC-Y}. Moreover, if Assumption \ref{asump1} holds, $g_{i}^{*}(Y_i)= X_i$ for all $1 \leq i\leq n$ is a solution of \eqref{eq:nmc-Y}.
\end{theorem}
\begin{proof}
A proof is presented in Section \ref{proof:nonlinear-gaussian}.
\end{proof}
In the multiple MC optimization \eqref{eq:multiple-MC-Y}, each variable $Y_i$ is assigned to $n-1$ transformation functions $\{g_{i,i'}^{*}(Y_i): 1\leq i'\neq i\leq n\}$. However, when variables $X_1$,...,$X_n$ are jointly Gaussian, all these functions are equal to $\pm X_i$. In general this is not true for non-Gaussian distributions. On the other hand, when the Assumption \ref{asump1} holds, a solution of the NMC optimization \eqref{eq:nmc-Y} recovers the sign of latent variables as well.

We define the matrices $\Lambda_{MultiMC}$ and $\Lambda_{NMC}$ by
\begin{align*}
\Lambda_{MultiMC}(i,i')&=\EE[g_{i,i'}^*(Y_i) g_{i',i}^*(Y_{i'})]\\
\Lambda_{NMC}(i,i')&=\EE[g_i^*(Y_i) g_{i'}^*(Y_{i'})].
\end{align*}
Moreover, we let $J_{MultiMC}=\Lambda_{MultiMC}^{-1}$ and $J_{NMC}=\Lambda_{NMC}^{-1}$. We define $G_{NMC}=(V_{NMC},E_{NMC})$ such that $(i,j)\in E_{NMC}$ if and only if $J_{NMC}(i,j)\neq 0$. We also let $\cG_{NMC}$ be the set of all possible $G_{NMC}$ since the solution of the optimization \eqref{eq:nmc-Y} may not be unique. Similarly, we define $G_{MultiMC}$ and $\cG_{MultiMC}$.

\begin{corollary}\label{cor:nonlinear-guassian-graphical-model}
Let $G_{Y}$ be a graphical model of variables $Y_i=f_i(X_i)$ according to Definition \ref{def:graphica-model}, where $f_i:\mathbb{R}\to\mathbb{R}$ is a bijective function, for $1\leq i\leq n$. Then $G_{Y}\in \cG_{MultiMC}$. Moreover, if $X_i$ $1\leq i\leq n$ satisfy Assumption \ref{asump1}, then we have $G_{Y}\in \cG_{NMC}$.
\end{corollary}
Corollary \ref{thm:nonlinear-guassian-graphical-model} characterizes the graphical model of variables $\{Y_i\}$ that are related to latent jointly Gaussian variables $\{X_i\}$ through the unknown bijective functions $\{f_i\}$. The family of distributions considered in this corollary is broad and includes many Gaussian distributions as well as distributions whose variables are bijective functions of Gaussian variables. Graphical models characterized in Corollary \ref{thm:nonlinear-guassian-graphical-model} can be used in computation of marginal distributions, computation of the mode of the joint distribution, and in other applications of estimation and prediction similarly to the case of Gaussian graphical models.

\begin{figure}[t]
  \centering
      \includegraphics[width=0.5\textwidth]{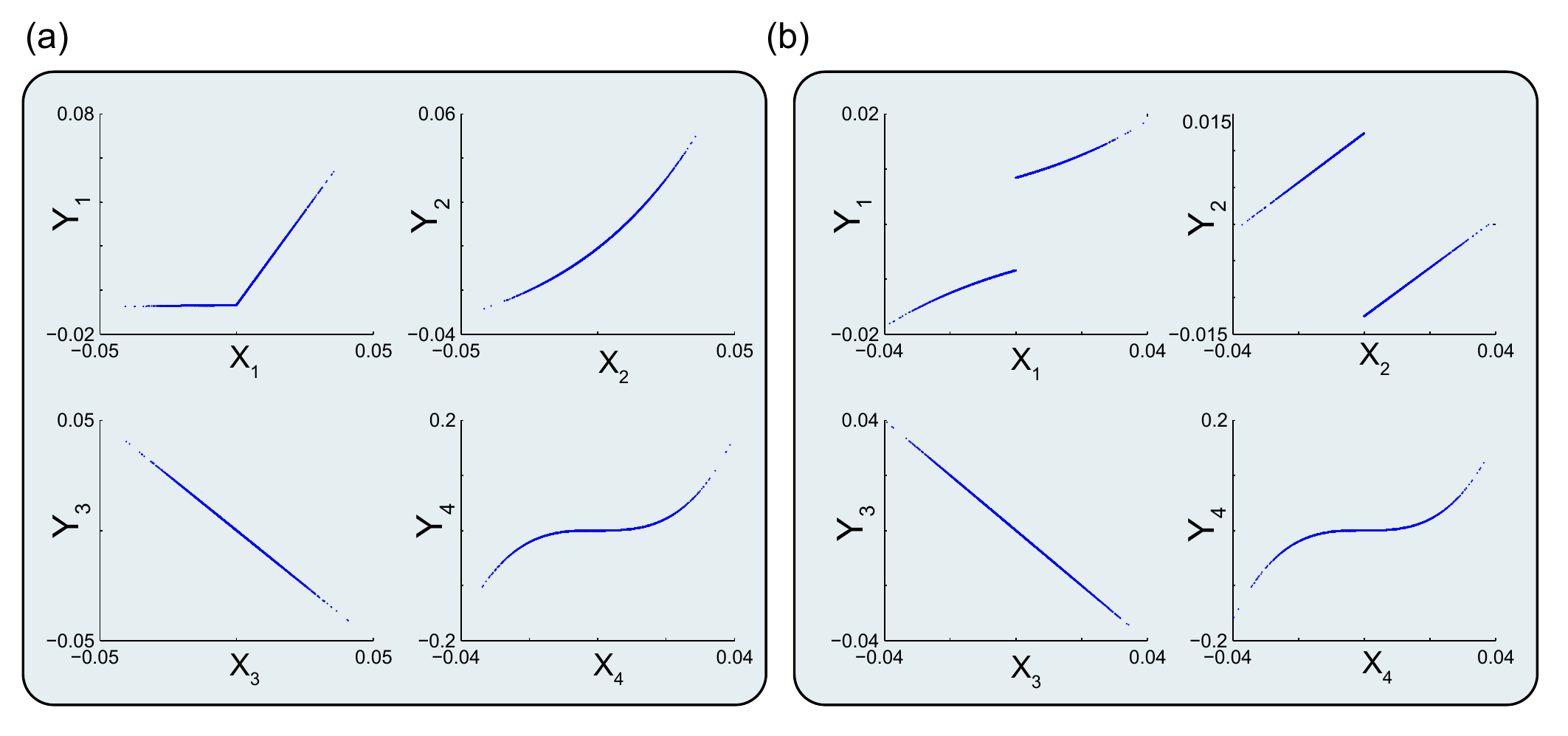}
  \caption{Relationships between jointly Gaussian variables $X_i$ and observations $Y_i$ for example 1 (panel a), and for example 2 (panel b).}
  \label{fig:x-y-phi}
\end{figure}

Next we provide two examples to highlight similarities and differences between our proposed inference framework (Theorem \ref{cor:nonlinear-guassian-graphical-model}) and the one in reference \cite{liu2012high}. Liu et al.  \cite{liu2012high} considers the graphical model inference problem for functions of jointly Gaussian variables when the underlying link functions are monotone. In \cite{liu2012high} to characterize the graphical model of variables $Y_i$ according to Definition \ref{def:graphica-model} (or equivalently the covariance matrix of jointly Gaussian variables $X_i$), rank statistics between variables and connections between Spearman's and Pearson's correlation coefficients are employed.

Consider four zero mean jointly Gaussian variables $X_1$,...,$X_4$ with the covariance matrix
\begin{align}
\Lambda_X=\begin{bmatrix}
1.0  &  0.4 &   0.2 &   0.3\\
0.4  &  1.0  &  0.3  &  0.2 \\
0.2  &  0.3  &  1.0 &   0.4\\
0.3  &  0.2  &  0.4  &  1.0
\end{bmatrix}.
\end{align}
In this case we have $J_X(1,3)\approx 0$ and $J_X(2,4)\approx 0$ (the underlying graphical model is illustrated in Figure \ref{fig:ex-graph}). We observe samples from $Y_i=f_i(X_i)$. In Example 1, we have
\begin{align}\label{SIeq:setup1}
Y_1&=f_1(X_1)= \begin{cases}
    10X_1,& \text{if } X_1\geq 0,\\
    \frac{1}{10} X_1,              & \text{otherwise},
\end{cases}, \quad Y_2=f_2(X_2)=e^{20X_2},\\
Y_3&=f_3(X_3)=-X_3,\quad Y_4=f_4(X_4)=X_4^3. \nonumber
\end{align}
In this example all link functions are continuous and monotone satisfying the model assumption of our proposed inference framework (Corollary \ref{cor:nonlinear-guassian-graphical-model}) and the copula method \cite{liu2012high}. In Example 2, we violate the model assumption for both our inference framework and for the copula method by considering non-monotone link functions:

\begin{align}\label{SIeq:setup2}
Y_1&=f_1(X_1)= \begin{cases}
    e^{20X_1},& \text{if } X_1\geq 0,\\
    -e^{-20X_1},   & \text{otherwise},
\end{cases}\\
Y_2&=f_2(X_2)= \begin{cases}
    \frac{1}{\max(X_2)}X_2-1,& \text{if } X_2\geq 0,\\
    \frac{-1}{\min(X_2)}X_2+1,   & \text{otherwise},
\end{cases}\nonumber\\
Y_3&=f_3(X_3)=-X_3,\quad Y_4=f_4(X_4)=X_4^3.\nonumber
\end{align}
Relationships between samples of these variables are illustrated in Figure \ref{fig:x-y-phi}. We have drawn 10,000 samples of these random variables. Note that all variables are normalized to have zero means and unit variances. The functions $f_i(\cdot)$ remain unknown for the inference part.

In the computation of NMC, if variables are continuous and we only observe samples from their joint distributions, the empirical computation of conditional expectations in Algorithm \ref{alg:nmc} may be challenging owing to the lack of sufficient samples. One approach to compute empirical conditional expectations at point $x_i\in \mathbb{R}$ is to use all samples in its $B_i$ neighborhood. With $m$ samples $\{( x_1^{(s)}, \dots, x_{n}^{(s)})\}_{s=1}^m$, for any $i, i'$, this leads to
\begin{align}
P^{(m)}_{X_i, X_{i'}}(x_i, x_{i'})  =\frac{1}{m} \sum_{s=1}^m \mathbf{1} \Big\{ &x_i^s \in \left[ x_i - \frac{B_i}{2}, x_i+ \frac{B_i}{2} \right],\\
&x_{i'}^s \in \left[ x_{i'} - \frac{B_{i'}}{2}, x_{i'}+ \frac{B_{i'}}{2} \right] \Big\}.\nonumber
\end{align}
In our ACE implementations to compute NMC for continuous variables we have considered both fixed and variable window sizes (i.e., $B_i$'s). In our simulations, in the variable window size case, we consider $10$ bins and choose $B_i$'s so that the number of samples in different bins are the same.

Let $\tilde{J}$ be an estimation of the inverse covariance matrix. Since in the underlying graphical model there are no edges between nodes 1 and 3, and nodes 2 and 4 (Figure \ref{fig:ex-graph}), we use $\left(|\tilde{J}(1,3)|+|\tilde{J}(2,4)|\right)/\sum_{1\leq i,j\leq n} |\tilde{J}(i,j)|$ as a measure to evaluate error in inferring the graphical model structure.

Figure \ref{fig:error-gaussian} shows inference errors for $J_{nmc}$ (i.e., the NMC inference framework), $J_{copula}$ (i.e., the copula method \cite{liu2012high}), $J_{Y}$ (i.e., observed variables) and $J_{X}$ (i.e., latent variables) for Example 1 and 2. In the case of Example 1, both NMC and copula inference frameworks have small and comparable errors. In the case of Example 2 where we violate the monotonicity assumption of link functions, NMC appears to outperform copula. Experiments have been repeated over 50 random realizations of variables.

In this section we focused on the application of MC and NMC in learning graphical models for functions of jointly Gaussian variables. A similar MC/NMC framework can potentially be useful in learning graphical models in other setups such as tree graphical models with incomplete samples \cite{choi2011learning}. We leave further exploration of this application for future research.

\begin{figure}[t]
  \centering
      \includegraphics[width=0.2\textwidth]{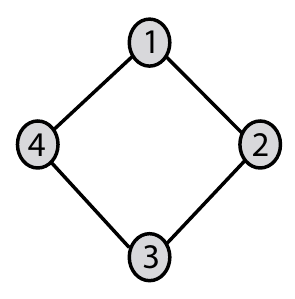}
  \caption{The underlying graphical model considered in Examples 1 and 2 in Section \ref{subsec:inference-guassian-nets}.}
  \label{fig:ex-graph}
\end{figure}

\begin{figure}[t]
  \centering
      \includegraphics[width=0.4\textwidth]{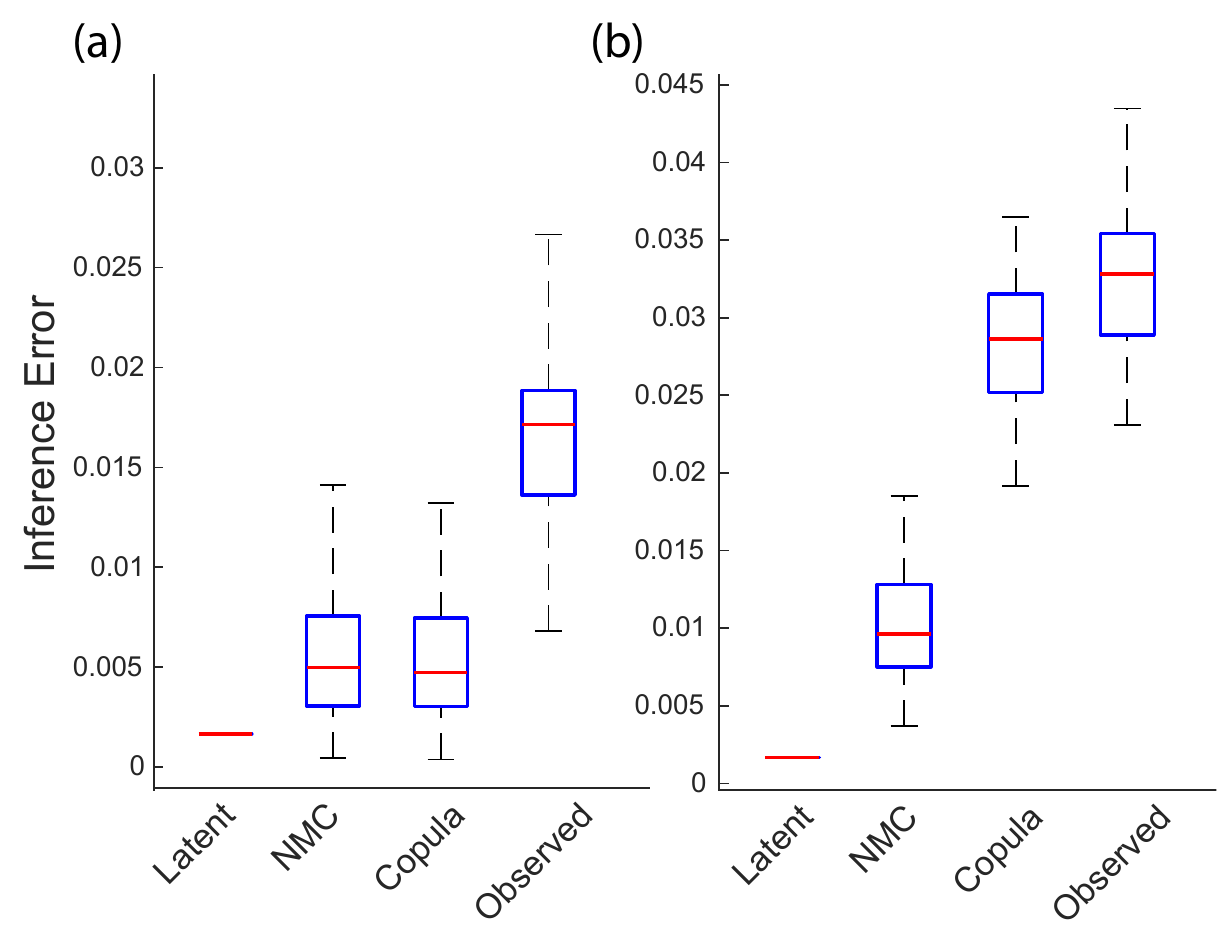}
  \caption{Inference errors for $J_{nmc}$ (i.e., the NMC inference framework), $J_{copula}$ (i.e., the copula method \cite{liu2012high}), $J_{Y}$ (i.e., observed variables) and $J_{X}$ (i.e., latent variables) for setups of Example 1 (panel a) and Example 2 (panel b). Experiments have been repeated over 50 random realizations of variables. The red line in the middle of each box is the sample median. The tops and bottoms of each box are the 25-th and 75-th percentiles of the samples, respectively. The performance of NMC and copula inference frameworks are comparable in the setup of Example 1, while in the one of Example 2 NMC outperforms copula.}
  \label{fig:error-gaussian}
\end{figure}

\subsection{Illustration of Sample NMC's Use in a Data Application}\label{sec:cancer}
Having illustrated applications of the NMC optimization in learning nonlinear dependencies among variables, here we provide an example to demonstrate NMC's utility in a real data application. We validate inference results with second, distinct data set that was not used in the inference part. This validation step provides evidence that inference results are likely to be meaningful.

{\bfseries Data:} Cancer is a complex disease involving abnormal cell growth with the potential to invade or spread to other parts of the body \cite{stewart2003world}. Different studies have shown associations of micro RNA patterns in different human cancers \cite{calin2006microrna,li2014potential}. In this section, we use normalized RNA sequence counts from the TCGA data portal for the Glioma cancer (GBMLGG) at the gene level \cite{li2014potential}. Other cancer types can be considered in our framework as well. We use the processed data provided in \cite{li2014potential}. Let the variable $X_i$ denote the RNA sequence counts of the gene $i$. We have samples from this variable in $m$ patients denoted by $\{x_i^j\}_{j=1}^{m}$. In the following we explain inference assumptions and steps we have taken in this application.

{\bfseries Inferring nonlinear gene-gene interactions:} for each cancer type, first we select the top $500$ highly-variant genes based on their normalized variances (i.e., $n=500$) \cite{huber2002variance}. Let $\phi_i^*(X_i)$ for $1\leq i\leq n$ be a solution for the NMC optimization \eqref{eq:nmc} over a complete graph without self-loops. Let $A_{NMC}\in\mathbb{R}^{n\times n}$ be a symmetric adjacency matrix where $A_{NMC}[i,i']=\EE[\phi_i^*(X_i) \phi_{i'}^*(X_{i'})]$ for $1\leq i,i'\leq n$. Moreover, let $A_{Lin}$ be a symmetric linear covariance matrix where $A_{Lin}[i,i']=\EE[X_i X_{i'}]$ for $1\leq i,i'\leq n$. Figure \ref{fig:cancer1}-a illustrates top $5\%$ of elements of the $|A_{NMC}|-|A_{Lin}|$ matrix for the Glioma Cancer (GBMLGG). The non-zero elements of this matrix represent gene pairs whose RNA sequence counts are strongly and nonlinearly associated to each other. In Figure \ref{fig:cancer1-details}, we consider some other density parameters as well (e.g., $1\%$ and $10\%$).

It is important to note that this real-data example is based on several heuristic steps and assumptions about the underlying data that we explain below.

{\bfseries Assumptions:} we assume that different RNA sequence count samples for a gene (i.e., $\{x_i^j\}_{j=1}^{m}$) are independent. That is, there is no between-patient correlations among these samples \cite{pritchard2000inference}. Also we assume that conditions of Theorem \ref{thm-nmc-guass-gen} holds. That is, input data comes from, possibly nonlinear, functions of some latent jointly Gaussian variables satisfying conditions of Theorem \ref{thm-nmc-guass-gen}. These functions are unknown and bijective. This assumption is less restrictive than the one of the standard covariance analysis where input variables are assumed to be jointly Gaussian.

{\bfseries Inferring Nonlinear Gene Modules:} we partition each network to $k$ groups using a standard spectral clustering algorithm based on the modularity transformation \cite{newman2006modularity}. We consider different values of $k$ (between 1 and 20) to obtain dense and large clusters. Then, in a heuristic manner, we merge small and heavily overlapping modules to form large and dense ones. We define {\it a gene module} as a group of genes that are densely connected to each other in the network. A gene module is called nonlinear if it is present in the NMC network but not in the linear one. We use a permutation test \cite{good2013permutation} to compute a $p$-value for each gene module in the network by permuting the network structure and comparing the density of the module in the original network with the ones in permutated ones. We only consider gene modules with $p$-values less than $0.05$. An example of such a nonlinear gene module module is illustrated by a yellow box in Figure \ref{fig:cancer1}-a.

{\bfseries Validations using Survival Time Analysis:} we partition individuals to two equal-size groups based on their average ranks of normalized RNA sequence counts in that module to evaluate if there are statistically significant differences between survival times of patients in the two groups. In order to do that, we perform a standard survival time analysis for each module based on Kaplan-Meier procedure to estimate the underlying survival function \cite{kaplan1958nonparametric}. We compute its associated log-rank $p$-value to determine its association with individual survival times in the considered cancer type \cite{bland2004logrank}. We perform Benjamini and Hochberg multiple hypothesis correction \cite{benjamini1995controlling} for the computed $p$-values of different nonlinear modules. We find that this gene module is significantly associated with survival times of cancer patients (Figure \ref{fig:cancer1}-b) while it is not detected using linear association measures. Several references \cite{bailey2009nonlinear,huang2007potential,lowengrub2010nonlinear,nagel1971computer,wagner1986steady} have hypothesized that complex nonlinear relationships among genes may play important roles in cancer pathways. The proposed NMC algorithm and inferred nonlinear gene modules can be used in discovering such complex nonlinear relationships in different cancer types. To substantiate these inferences, further experiments should be performed to determine the involvement of these nonlinear gene interactions and modules in different cancers, which is beyond the scope of the present paper.
\begin{figure}[t]
  \centering
      \includegraphics[width=9cm]{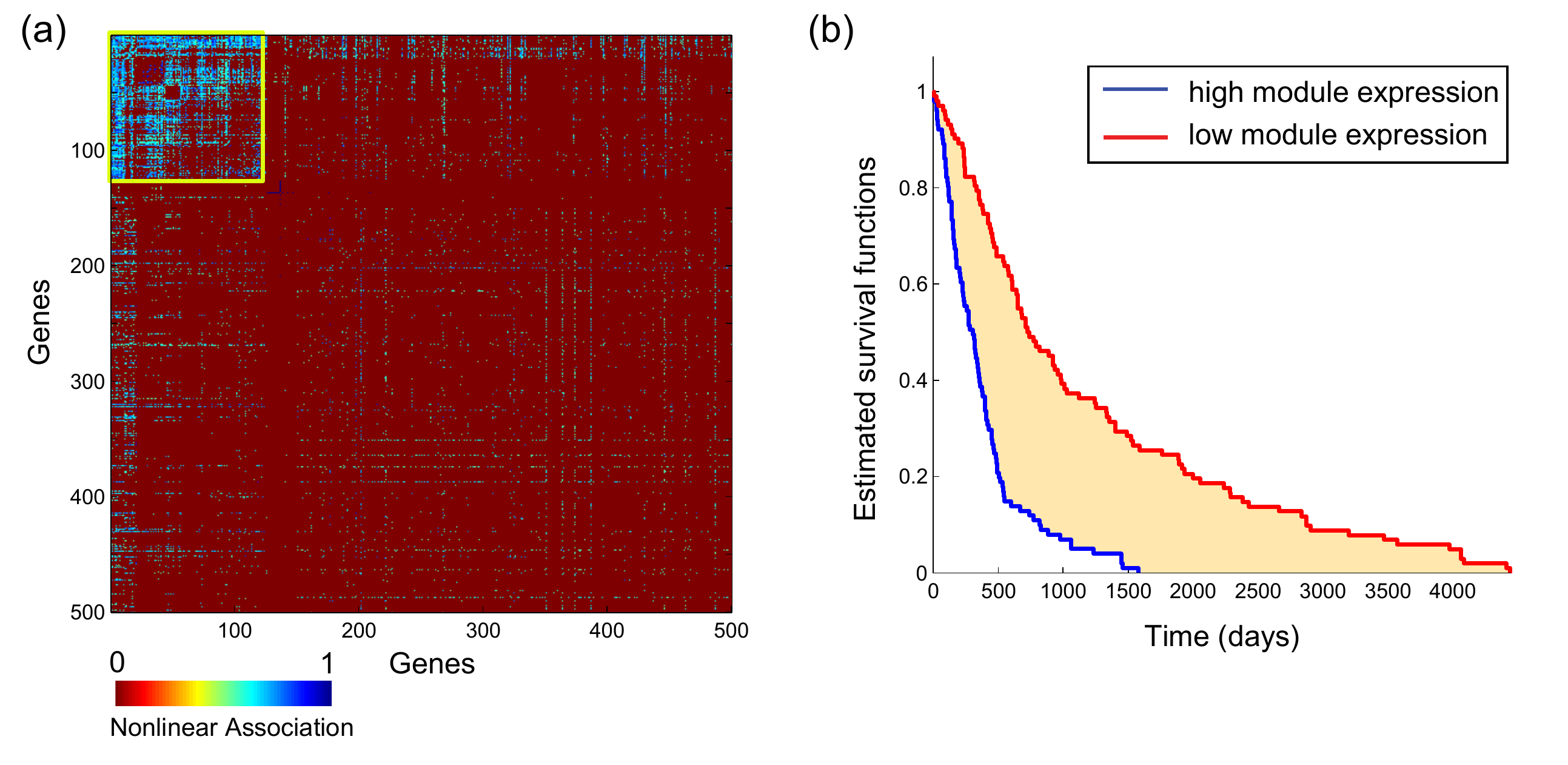}
      \caption{(a) A nonlinear gene module of Glioma cancer (GBMLGG) defined in Section \ref{sec:cancer}, as a group of genes whose RNA sequence counts are strongly and nonlinearly dependent among each other. (b) Survival time curves for the corresponding nonlinear cancer module. Survival times of cancer patients are significantly associated with average normalized RNA sequence counts of inferred gene module.}
  \label{fig:cancer1}
\end{figure}

\begin{figure}[t]
  \centering
      \includegraphics[width=9cm]{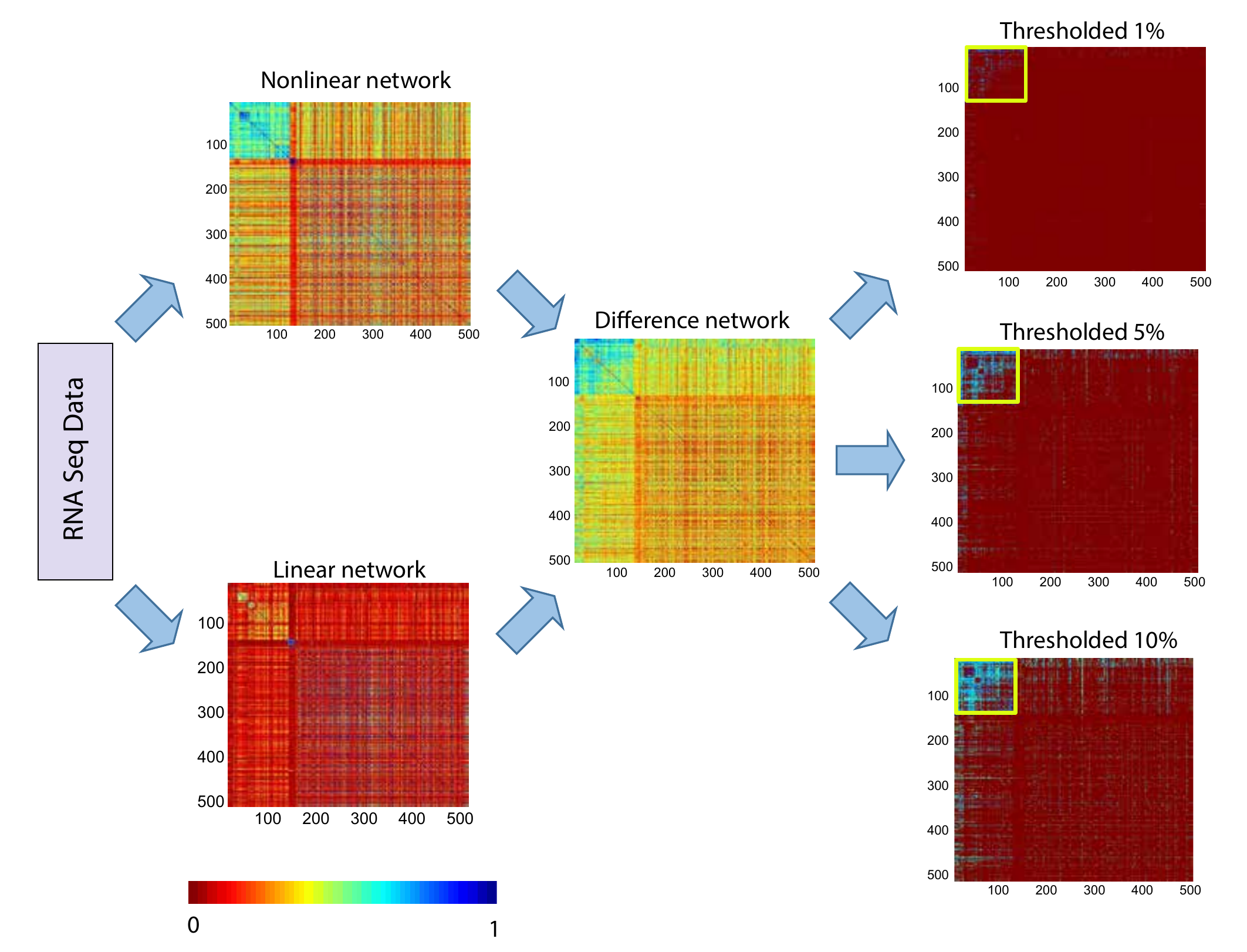}
      \caption{An example illustrating the framework considered in Section \ref{sec:cancer} for Glioma (GBMLGG). Using normalized RNA sequence counts from the TCGA data portal, we construct both linear ($A_{Lin}$) and nonlinear ($A_{NMC}$) complete dependency graphs between genes. We then select the top $5\%$ interactions among genes in the NMC network with the largest nonlinear association increases compared to their linear association strengths. Then, using spectral clustering \cite{newman2006modularity} along with a heuristic step to merge small overlapping clusters, we identify dense and large clusters over the thresholded network. In this example, we illustrate the inferred nonlinear gene module over networks with density parameters $1\%$, $5\%$ and $10\%$.}
  \label{fig:cancer1-details}
\end{figure}

\section{Proofs}\label{Sec:Proofs}

\subsection{Proof of Theorem \ref{thm:hilbert-nmc}}\label{subsec:proof-hilbert-nmc}
\begin{proof}
Recall that $ \{ \psi_{i,j} \}_{j=1}^{\infty}$ is the orthonormal basis of $H_{X_i}$ for $1\leq i\leq n$. We can represent functions $\phi_i$ and $\phi_{i'}$ in terms of the basis functions, i.e., for any $x_i \in \mathcal{X}_i, x_{i'} \in \mathcal{X}_{i'}$,
\begin{align*}
& \phi_i(x_i) = \sum_{j=1}^{\infty} a_{i,j} \psi_{i,j} (x_i), \\
& \phi_{i'}(x_{i'}) = \sum_{j=1}^{\infty} a_{i',j} \psi_{i',j} (x_{i'}),
\end{align*}
for two sequences of coefficients $\{a_{i, j}\}_{j=1}^{\infty}$ and $\{a_{i', j}\}_{j=1}^{\infty}$.
Thus, the constraint $\EE[\phi_k(X_k)^2]=1$ in optimization \eqref{eq:nmc} would be translated into $\sum_{j=1}^{\infty} a_{k,j}^2 =1$ and the constraint $\mathbb{E}[\phi_k(X_k)]=0$ is simplified to $\sum_{j=1}^{\infty} a_{k,j} \mathbb{E}[\psi_{k,j} (X_k)]=0$, for $k=i, i'$. Moreover, we have
\begin{align}
\mathbb{E}[\phi_i(X_i)\phi_{i'}(X_{i'})]= \sum_{j, j'=1}^{\infty} a_{i,j} a_{i',j'}\ \mathbb{E}[\psi_{i,j} (X_i) \psi_{i',j'} (X_{i'})].
\end{align}
\noindent
Thus, Network Maximal Correlation optimization \eqref{eq:nmc} can be re-written as follows:
\begin{align}\label{eq:proof-hilbert-opt1}
\sup_{\{\{a_{i,j}\}_{i=1}^{n}\}_{j=1}^{\infty}} \quad &\sum_{i,i'}\sum_{j, j'} a_{i,j} a_{i',j'}\ \mathbb{E}[\psi_{i,j} (X_i) \psi_{i',j'} (X_{i'})] \\
& \sum_{j=1}^{\infty} a_{i,j}^2 =1, ~~ 1\leq i\leq n, \nonumber \\
& \sum_{j=1}^{\infty} a_{i,j}\ \mathbb{E}[\psi_{i,j} (X_i)]=0, ~~ 1\leq i\leq n.\nonumber
\end{align}
\noindent
Moreover, $\{\psi_{i,j} \psi_{i',j'}\}_{j,j'}$ form a basis for the following set of functions $\{\phi_{i,i'}: \mathcal{X}_i \times \mathcal{X}_{i'} \to \mathbb{R}~:~ \mathbb{E}[\phi_{i,i'}^2(X_i, X_{i'})] < \infty\}$. Since $P_{X_i, X_{i'}}(\cdot,\cdot) \in $ belongs to this set of functions, we can write
\begin{align}\label{eq:proof-hilbert-joint}
P_{X_i, X_{i'}} (x_i, x_{i'})= \sum_{j, j'} \rho_{i,i'}^{j,j'} \psi_{i,j} (x_i) \psi_{i',j'} (x_{i'}).
\end{align}
This completes the proof.
\end{proof}
\noindent

\subsection{Proof of Theorem \ref{thm:NMC2MCP}}\label{proof:thm:nmc2mcp}
\begin{proof}
To prove Theorem \ref{thm:NMC2MCP}, first we show that the constraints $\mathbb{E}[\phi_i(X_i)]=0$ and $\mathbb{E}[\phi_i(X_i)^2]=1$ lead to the same solution as the constraints $\text{var} (\phi_i(X_i)) =1$. We then show that the constraint $\mathbf{a}_i \perp \sqrt{\mathbf{p}_i}$ can be incorporated into the objective function, without changing the solution.
\begin{lemma}\label{lemma:NMC_reformulate_variance}
The NMC optimization \eqref{eq:nmc} can be written as follows:
\begin{align}\label{eq:NMC-variance-Reformulate}
\max_{\phi_1, \dots, \phi_n}\quad & \sum_{(i, i') \in E} \mathbb{E}[(\phi_i (X_i)- \bar{\phi_i})(\phi_{i'} (X_{i'})- \bar{\phi_{i'}})]\\
&\text{var}(\phi_i(X_i))=1, \quad 1\leq i\leq n,\nonumber
\end{align}
where $\bar{\phi_i}$ and $\text{var }(\phi_i(X_i))$ represent the mean and the variance of the random variable $\phi_i(X_i)$.
\end{lemma}
\begin{proof}
In this proof, we denote the random variable $\phi_i(X_i)$ by $\phi_i$. We let the optimal objective value of optimization \eqref{eq:NMC-variance-Reformulate} be $\tilde{\rho}_G$. We also let $\phi_i^*$ be an solution of \eqref{eq:nmc}. The set of functions $\phi_i^*$ for $i=1, \dots, n$ is feasible for optimization \eqref{eq:NMC-variance-Reformulate} and therefore we have $\rho_G \le \tilde{\rho}_G $. On the other hand, let $\phi_i^{**}$ be an solution of optimization \eqref{eq:NMC-variance-Reformulate}. Let $\tilde{\phi}_i= \phi_i^{**}- \bar{\phi_i}^{**}$. The set of functions $\tilde{\phi}_i$ for $i=1, \dots, n$ is feasible for optimization \eqref{eq:nmc}. Thus, we have $\rho_G \ge \tilde{\rho}_G$. Therefore, we have that $\rho_G = \tilde{\rho}_G$.
\end{proof}
\noindent
We have
\begin{align}
1& = \mathbb{E}[\phi_i(X_i)^2]- (\mathbb{E}[\phi_i(X_i)])^2= ||\mathbf{a}_i||_2^2- (\mathbf{a}_i \sqrt{\mathbf{p}_i})^2\\
& = \mathbf{a}_i^T \left( I - \sqrt{\mathbf{p}_i} \sqrt{\mathbf{p}_i}^T \right) \mathbf{a}_i.\nonumber
\end{align}
We next show that the matrix $I - \sqrt{\mathbf{p}_i} \sqrt{\mathbf{p}_i}^T $ is positive semidefinite and the only vectors in its null space are $\mathbf{0}$ and $\sqrt{\mathbf{p}_i}$. This is because
\begin{align}\label{eq:sdp-proof}
\mathbf{x}^T \left( I - \sqrt{\mathbf{p}_i} \sqrt{\mathbf{p}_i}^T \right) \mathbf{x} = ||\bx||_2^2- (\mathbf{x} \sqrt{\mathbf{p}_i})^2 \ge 0,
\end{align}
where we use Cauchy-Schwartz and $||\sqrt{\mathbf{p}_i}||_2^2=1$ to obtain the last inequality \eqref{eq:sdp-proof}. This inequality becomes an equality if and only if $\mathbf{x}=0$ or $\mathbf{x}= \sqrt{\mathbf{p}_i}$.
\\Now consider the objective function of optimization \eqref{eq:NMC-variance-Reformulate}. We have
\begin{align*}
& \mathbb{E}[(\phi_i(X_i)- \bar{\phi_i})(\phi_{i'}(X_{i'})- \bar{\phi_{i'}})]= \mathbb{E}[\phi_i(X_i) \phi_{i'}(X_{i'})] - \bar{\phi_i} \bar{\phi_{i'}} \nonumber \\
& = \mathbf{a}_i^T Q_{i,i'} \mathbf{a}_{i'}- (\mathbf{a}_i^T \sqrt{\mathbf{p}_i}) (\mathbf{a}_{i'}^T \sqrt{\mathbf{p}_{i'}})= \mathbf{a}_i^T \left( Q_{i,i'}- \sqrt{\mathbf{p}_i} \sqrt{\mathbf{p}_{i'}}^T \right) \mathbf{a}_{i'}.\nonumber
\end{align*}
Therefore, optimization \eqref{eq:NMC-variance-Reformulate} (which is equivalent to the NMC optimization \eqref{eq:nmc} according to Lemma \ref{lemma:NMC_reformulate_variance}) can be written as,
\begin{align}\label{eq:nmc-mcp-proof-opt}
\max_{\mathbf{a}_i} \quad & \sum_{(i, i') \in E} \mathbf{a}_i^T \left( Q_{i,i'}- \sqrt{\mathbf{p}_i} \sqrt{\mathbf{p}_{i'}}^T \right) \mathbf{a}_{i'}\\
& \mathbf{a}_i^T \left( I - \sqrt{\mathbf{p}_i} \sqrt{\mathbf{p}_i}^T \right) \mathbf{a}_i=1, \quad 1\leq i\leq n.\nonumber
\end{align}
For each $i$, since $ I - \sqrt{\mathbf{p}_i} \sqrt{\mathbf{p}_i}^T $ is positive semidefinte. Thus, we can write $I - \sqrt{\mathbf{p}_i} \sqrt{\mathbf{p}_i}^T= B_i B_i^T$. Recall $\mathbf{b}_i= B_i \mathbf{a}_i$. Thus, constraints of optimization \eqref{eq:nmc-mcp-proof-opt} can be written as $\mathbf{b}_i^T \mathbf{b}_i=||\mathbf{b}_i||_2^2=1$. We next write $\mathbf{a}_i$ as a function of $\mathbf{b}_i$. Note that since $B_i$ is not invertible, there are many choices for $\mathbf{a}_i$ as a function of $\mathbf{b}_i$ characterized as follows:
\noindent
Let $U_i \Sigma_i U_i^T $ be the singular value decomposition of $B_i$. The vector $\sqrt{\mathbf{p}_i}$ is the singular vector corresponding to singular value zero ($\sigma_{i}^{(1)}=0$).
\begin{align}\label{eq:ba-i-choices}
\mathbf{a}_i = &\big( [U^{(2)}_i, \dots, U^{(|\mathcal{X}_i|)}_i] \text{diag}(1/ \sigma_i^{(2)}, \dots, 1/ \sigma_i^{(n_i)})\\ &[U^{(2)}_i, \dots, U^{(|\mathcal{X}_i)}_i]^T \big) \mathbf{b}_i+ \alpha_i \sqrt{\mathbf{p}_i}= A_i \mathbf{b}_i + \alpha_i  \sqrt{\mathbf{p}_i},\nonumber
\end{align}
where $\alpha_i$ can be any scalar.\footnote{Since $B_i$ is symmetric it has a set of $|\mathcal{X}_i|$ orthonormal eigenvectors and can be written as \[B_i= \sum_{j=1}^{|\mathcal{X}_i|} \mathbf{v}_j \sigma_i^{(j)} \mathbf{v}_j^T.\] We have $\mathbf{b}_i= \sum_{j=2}^{|\mathcal{X}_i|} \beta_j \mathbf{v}_j$ and $\mathbf{a}_i= \sum_{j=1}^{|\mathcal{X}_i|} \alpha_j \mathbf{v}_j$. From $\mathbf{b}_i= B_i \mathbf{a}_i$, we obtain that $\alpha_j= \beta_j/\sigma_i^{(j)}$ for $j \ge 2$, where $\alpha_1$ can be any scalar.} Below, we show that all choices of $\ba_i$ according to \eqref{eq:ba-i-choices} lead to the same objective function of optimization \eqref{eq:nmc-mcp-proof-opt}:
\begin{align*}
& \mathbf{a}_i^T \left( Q_{i,i'}- \sqrt{\mathbf{p}_i} \sqrt{\mathbf{p}_{i'}}^T \right) \mathbf{a}_{i'}  =  \mathbf{b}_i^T A_i^T \left( Q_{i,i'}- \sqrt{\mathbf{p}_i} \sqrt{\mathbf{p}_{i'}}^T \right)  A_{i'} \mathbf{b}_{i'}  \\
&  + \mathbf{b}_i^T A_i^T \left( Q_{i,i'}- \sqrt{\mathbf{p}_i} \sqrt{\mathbf{p}_{i'}}^T \right)  \alpha_{i'} \sqrt{\mathbf{p}_{i'}}  \\
& + \alpha_{i} \sqrt{\mathbf{p}_{i}}^T \left( Q_{i,i'}- \sqrt{\mathbf{p}_i} \sqrt{\mathbf{p}_{i'}}^T \right)  A_{i'} \mathbf{b}_{i'} \\
& + \alpha_{i} \sqrt{\mathbf{p}_{i}}^T \left( Q_{i,i'}- \sqrt{\mathbf{p}_i} \sqrt{\mathbf{p}_{i'}}^T \right)  \alpha_{i'} \sqrt{\mathbf{p}_{i'}}\\
& \stackrel{(1)}{=} \mathbf{b}_i^T A_i^T \left( Q_{i,i'}- \sqrt{\mathbf{p}_i} \sqrt{\mathbf{p}_{i'}}^T \right)  A_{i'} \mathbf{b}_{i'}  \\
& + \mathbf{b}_i^T A_i^T Q_{i,i'} \alpha_{i'} \sqrt{\mathbf{p}_{i'}}   + \alpha_{i} \sqrt{\mathbf{p}_{i}}^T Q_{i,i'} A_{i'} \mathbf{b}_{i'}   \\
& \stackrel{(2)}{=}  \mathbf{b}_i^T A_i^T \left( Q_{i,i'}- \sqrt{\mathbf{p}_i} \sqrt{\mathbf{p}_{i'}}^T \right)  A_{i'} \mathbf{b}_{i'}   + \mathbf{b}_i^T A_i^T \sqrt{\mathbf{p}_{i}}  \alpha_{i'}  \\
& + \alpha_{i} \sqrt{\mathbf{p}_{i'}}^T A_{i'} \mathbf{b}_{i'} =  \mathbf{b}_i^T A_i^T \left( Q_{i,i'}- \sqrt{\mathbf{p}_i} \sqrt{\mathbf{p}_{i'}}^T \right)  A_{i'} \mathbf{b}_{i'},
\end{align*}
where $(1)$ follows from $\sqrt{\mathbf{p}_{i'}}^T \sqrt{\mathbf{p}_{i'}}=1$, $\sqrt{\mathbf{p}_{i}}^T \sqrt{\mathbf{p}_{i}}=1$, $\sqrt{\mathbf{p}_i} Q_{i,i'}\sqrt{\mathbf{p}_{i'}}=1$, $A_i^T \sqrt{\mathbf{p}_i}=0$, and $\sqrt{\mathbf{p}_{i'}}^T A_{i'}=0$; and $(2)$ follows from $\sqrt{\mathbf{p}_i}^T Q_{i,i'}= \sqrt{\mathbf{p}_{i'}}^T$ and  $Q_{i,i'} \sqrt{\mathbf{p}_{i'}}= \sqrt{\mathbf{p}_{i}}$. Therefore, the NMC optimization \eqref{eq:nmc} can be written as
\begin{align*}
\max_{\mathbf{b}_1, \dots, \mathbf{b}_n}\quad & \sum_{(i, i') \in E} \mathbf{b}_i^T A_i^T \left( Q_{i,i'}- \sqrt{\mathbf{p}_i} \sqrt{\mathbf{p}_{i'}}^T \right)  A_{i'} \mathbf{b}_{i'} \\
&||\mathbf{b}_i||_2=1 \qquad 1 \le i \le n.
\end{align*}
\end{proof}

\subsection{Proof of Proposition \ref{prop:jump-local}}\label{proof-prop-jump-local}
\begin{proof}
Let $\barb$ be a solution of the MEP \eqref{eq:MEPformulation}. Suppose that there exists an $1 \le i \le n$ such that $\lambda_i(\barb) < 0 $. Let $\hatb=(\hatb_1,\ldots,\hatb_n)$ be defined as: $\hatb_i=-\barb_i$, for any $i$ such that $\lambda_i(\barb) < 0$, and $\hatb_{i'}=\barb_{i'}$, otherwise. We show that by flipping the sign of $\barb_i$ while keeping the rest of $\barb_{i'}$ for $i'\neq i$ the same, $r(\cdot)$ increases. To show this we have

\begin{align}
r(\hatb)-r(\barb)& \stackrel{(1)}{=} 2\hatb_{i}^T \sum_{i'=1}^n C_{i,i'} \hatb_{i'}- 2\barb_{i}^T \sum_{i'=1}^n C_{i,i'} \barb_{i'}\nonumber\\
& \stackrel{(2)}{=} 2 (\hatb_{i}^T-\barb_{i}^T) \sum_{i'=1}^n C_{i,i'} \barb_{i'}\nonumber\\
& \stackrel{(3)}{=} -2 \barb_{i}^T \sum_{i'=1}^n C_{i,i'} \barb_{i'}\nonumber\\
& \stackrel{(4)}{=} -2 \lambda_i(\barb_{i})>0, \nonumber
\end{align}

where equations (1) and (4) follow from \eqref{eq:r-lam-def}, equation (2) follows from the fact that $\hatb_{i'}=\barb_{i'}$ for $i'\neq i$, and equation (3) follows from the fact that $\hatb_i=-\barb_i$. This completes the proof.
\end{proof}

\subsection{Proof of Theorem \ref{th:approxnmc}} \label{proof:th:approxnmc}
\begin{proof}
For any realization of the partitioning, consider NMC over all sub-graphs $G_m$ ($1 \le m \le M$) and denote the corresponding functions by $\phi_i^*$ for $1 \le i \le n$. We have
\begin{align*}
& \rho_G= \sum_{(i, i') \in E}  \mathbb{E}[{\phi}^*_i(X_i) {\phi}^*_{i'}(X_{i'})] \\
& = \sum_{(i, i') \in E \setminus E^c } \mathbb{E}[{\phi}^*_i(X_i) {\phi}^*_{i'}(X_{i'})]  + \sum_{(i, i') \in E^c }  \mathbb{E}[{\phi}^*_i(X_i) {\phi}^*_{i'}(X_{i'})] \\
& \stackrel{(1)}{=} \sum _{m=1}^M \sum_{(i, i') \in E_m } \mathbb{E}[{\phi}^*_i(X_i) {\phi}^*_{i'}(X_{i'})]  + \sum_{(i, i') \in E^c }  \mathbb{E}[{\phi}^*_i(X_i) {\phi}^*_{i'}(X_{i'})] \\
& \stackrel{(2)}{\le} \sum _{m=1}^M \hat{\rho}_{G_m} + \sum_{(i, i') \in E^c } \mathbb{E}[{\phi}^*_i(X_i) {\phi}^*_{i'}(X_{i'})]  \\
& =  \hat{\rho}_G +  \sum_{(i, i') \in E} \mathbf{1}\{(i, i') \in E^c\} \mathbb{E}[{\phi}^*_i(X_i) {\phi}^*_{i'}(X_{i'})],
\end{align*}
where equation (1) comes from the graph partitioning Definition \ref{def:partition}, and inequality (2) comes from the fact that $\hat{\rho}_{G_m}$ is the NMC for the partition ${G}_m=(V_m, E\cap (V_m \times V_m))$. Therefore, by taking expectation over the partitioning, we obtain
\begin{align*}
\rho_G \le \mathbb{E}[\hat{\rho}_G] + \epsilon \rho_G,
\end{align*}
which gives us
\begin{align*}
(1-\epsilon) \rho_G \le \mathbb{E}[\hat{\rho}_G ].
\end{align*}
\end{proof}

\subsection{Proof of Theorem \ref{thm:continiousnmc}} \label{proof:thm:continiousnmc}
\begin{proof}
We first present a lemma in which we bound the difference between $Q$ matrices, according to Definition \ref{def:Q-matrix}, by the difference between probability distributions. We then use a sensitivity analysis of the optimization whose optimum is NMC.
\begin{lemma}\label{lem:diffQbasedondiffP}
Let $P$ and $\tilde{P}$ be the matrix form of two joint probability distribution on $(\cX_i, \cX_{i'})$, such that $P_{X_i,X_{i'}}(x_i, x_{i'})=[P]_{x_i,x_{i'}}$ and $\tilde{P}_{X_i,X_{i'}}(x_i, x_{i'})=[\tilde{P}]_{x_i,x_{i'}}$ and where $D_{X_i}(P)=\text{diag}\left( P_{X_i}(x_i)~:~ x_i \in \mathcal{X}_i \right)$. We can bound the difference between $Q_{i,i'}$ and $\tilde{Q}_{i,i'}$ by the difference between $P$ and $\tilde{P}$, as follows:
\begin{align}\label{eq:lem-diff-Q}
& ||Q_{i,i'} - \tilde{Q}_{i,i'}||_2   \le
  \frac{1}{2 \delta^2} \sqrt{K} ||D_{X_{i}}({P})-D_{X_{i}}({\tilde{P}})||_{\infty}  \nonumber \\
  & + \frac{1}{2 \delta^2} \sqrt{K} ||D_{X_{i'}}({P})-D_{X_{i'}}({\tilde{P}})||_{\infty} + \frac{1}{\delta} \sqrt{K} ||{P}-{\tilde{P}}||_{\infty} ,
\end{align}
where $\delta$ is the minimum probability of all elements of $\mathcal{X}_i$ and $\mathcal{X}_{i'}$, under both $P$ and $\tilde{P}$, and $K$ is the maximum alphabet size. In particular, if $||P_{X_i, X_{i'}}- \tilde{P}_{X_i, X_{i'}}||_{\infty} \le \gamma$ and $K = \max\{ |\mathcal{X}_i|, |\mathcal{X}_{i'}|\}$, then we have
\begin{align}\label{eq:bound-Q-diff-clean}
||Q_{i,i'} - \tilde{Q}_{i,i'}||_2 \le 2 \gamma K^n \frac{1}{\delta^2}.
\end{align}
\end{lemma}
\begin{proof}
We have
\begin{align}
&||Q_{i,i'}-\tilde{Q}_{i,i'}||_2 \nonumber\\
&=\large{||}D_{X_i}({P})^{-\frac{1}{2}} {P} D_{X_{i'}}({P})^{-\frac{1}{2}}-D_{X_i}({\tilde{P}})^{-\frac{1}{2}} {\tilde{P}} D_{X_{i'}}({\tilde{P}})^{-\frac{1}{2}}\large{||}_2.\nonumber
\end{align}
We next insert the terms $\pm D_{X_i}({P})^{-\frac{1}{2}}{P} D_{X_{i'}}({\tilde{P}})^{-\frac{1}{2}}$ and $\pm D_{X_i}({P})^{-\frac{1}{2}}\tilde{P} D_{X_{i'}}({\tilde{P}})^{-\frac{1}{2}}$, pair the terms, and use triangle inequality to obtain
\begin{align}\label{Eq:bounddiffQs}
& ||Q_{i,i'}-\tilde{Q}_{i,i'}||_2  \le  ||D_{X_i}({P})^{-\frac{1}{2}}{P}||_2 ||D_{X_{i'}}({P})^{-\frac{1}{2}}-D_{X_{i'}}({\tilde{P}})^{-\frac{1}{2}}||_2 \\
& +||D_{X_{i'}}({\tilde{P}})^{-\frac{1}{2}}||_2 ||{\tilde{P}}||_2 ||D_{X_i}({P})^{-\frac{1}{2}}-D_{X_i}({\tilde{P}})^{-\frac{1}{2}}||_2 \nonumber \\
& + ||D_{X_{i'}}({\tilde{P}})^{-\frac{1}{2}}||_2 ||D_{X_i}({P})^{-\frac{1}{2}}||_2 ||{P}-{\tilde{P}}||_2 , \nonumber
\end{align}
where each term depends on the difference between $P$ and $\tilde{P}$. We have
\begin{align}\label{eq:bound-cauchy1}
||D_{X_i}({P})^{-\frac{1}{2}}{P}||_2 & \stackrel{(1)}{\le} \frac{1}{\sqrt{\delta}} \sqrt{K}\\
||D_{X_{i'}}({P})^{-\frac{1}{2}}-D_{X_{i'}}({\tilde{P}})^{-\frac{1}{2}}||_2 & \leq ||D_{X_{i'}}({P})^{-\frac{1}{2}}-D_{X_{i'}}({\tilde{P}})^{-\frac{1}{2}}||_{\infty}\nonumber\\
&\stackrel{(2)}{\le} ||D_{X_{i'}}({P})-D_{X_{i'}}({\tilde{P}})||_{\infty} \frac{1}{2 \delta^{3/2}},\nonumber
\end{align}
where inequality (1) comes from the Cauchy-Schwartz inequality and the fact that $||P||_{2} \leq \sqrt{K} ||P||_{1} \leq \sqrt{K}$. Inequality (2) follows from the fact that for $x,y \in \mathbb{R}$, we have
\begin{align}\label{eq:one-over-square-bound}
\frac{1}{\sqrt{x}}- \frac{1}{\sqrt{y}} = \frac{y-x}{\sqrt{x}\sqrt{y}(\sqrt{x}+ \sqrt{y})}\le |x-y| \frac{1}{2 \left(\min\{x, y\} \right)^{3/2}}.
\end{align}
Using \eqref{eq:bound-cauchy1} in \eqref{Eq:bounddiffQs} results in \eqref{eq:lem-diff-Q}.

Moreover we have

\begin{align}\label{eq:bound-marginal1}
&||D_{X_{i}}({P})-D_{X_{i}}({\tilde{P}})||_{\infty}\\
&\leq \sum_{\{X_1, \dots, X_n\}\setminus \{X_i\}} ||P_{X_1, , \dots, X_{n}}- \tilde{P}_{X_1, \dots, X_{n}}||_{\infty} \le \gamma K^{n-1}.\nonumber
\end{align}

Using \eqref{eq:bound-marginal1} and the fact that $K>1$ in \eqref{eq:lem-diff-Q} results in \eqref{eq:bound-Q-diff-clean}. This completes the proof of this Lemma.

\end{proof}
Let $P$ and $\tilde{P}$ be two distributions on $\mathcal{X}_1 \times \dots \times \mathcal{X}_n$. We shall compare the solution of the two following optimization problems.
\begin{align}\label{eq:temp1cont}
\max_{\mathbf{a}_i}\quad & \sum_{(i, i')\in E} \mathbf{a}_i^T Q_{i, i'} \mathbf{a}_i \\
& ||\mathbf{a}_i||_2 =1 , ~~~ 1 \le i \le n,\nonumber\\
& \mathbf{a}_i \perp \sqrt{\mathbf{p}_i}, ~~~ 1 \le i \le n,\nonumber
\end{align}
and
\begin{align}\label{eq:temp2cont}
\max_{\mathbf{a}_i}\quad & \sum_{(i, i')\in E} \mathbf{a}_i^T \tilde{Q}_{i, i'} \mathbf{a}_i \\
& ||\mathbf{a}_i||_2 =1 , ~~~ 1 \le i \le n,\nonumber\\
& \mathbf{a}_i \perp \sqrt{\tilde{\mathbf{p}}_i}, ~~~ 1 \le i \le n.\nonumber
\end{align}
Let $\rho_G$ and $\tilde{\rho}_G$ be the optimal values for \eqref{eq:temp1cont} and \eqref{eq:temp2cont}, respectively.

Recall that \[\delta=\min_{1 \le i \le n} \left( \min \{\delta_{X_i}(P), \delta_{X_i}(\tilde{P})\} \right).\] Suppose $\mathbf{a}^*_i$ yields the optimum of optimization \eqref{eq:temp1cont}. Based on this solution, we shall construct a feasible solution for optimization \eqref{eq:temp2cont} and then evaluate its objective function. For any $i$, let

\begin{align*}
\mathbf{c}_i= \frac{\mathbf{a}^*_i+\nu_i} {||\mathbf{a}^*_i+ \nu_i||_2},
\end{align*}
where $\nu_i=\sqrt{\tilde{\bp}_i} <\mathbf{a}^*_i, \sqrt{\mathbf{p}_i}- \sqrt{\tilde{\mathbf{p}}_i}>$. Under the assumption of Theorem \ref{thm:continiousnmc} (i.e., $\gamma \le \delta^{3/2} \mathcal{K}^{-1}$), we show that $||\nu_i||\leq 1/2$. We have
\begin{align}\label{eq:norm-nu-small}
||\nu_i|| & =||\sqrt{\tilde{\bp}_i} <\mathbf{a}^*_i, \sqrt{\mathbf{p}_i}- \sqrt{\tilde{\mathbf{p}}_i}>||_2 \stackrel{(3)}{\le} ||\sqrt{\tilde{\bp}_i}||_2 ||\mathbf{a}^*_i||_2 ||\sqrt{\mathbf{p}_i}- \sqrt{\tilde{\mathbf{p}}_i}||_2 \\
&\stackrel{(4)}{\le} ||\mathbf{p}_i- \tilde{\mathbf{p}}_i||_{\infty} \frac{1}{2\delta^{3/2}}\nonumber\\
&\stackrel{(5)}{\le} \frac{\gamma K^{n-1} }{2 \delta^{3/2}}\\
&\stackrel{(6)}{\le} \frac{1}{2},\nonumber
\end{align}
where Inequality (3) comes from the Cauchy-Schwartz, Inequality (4) comes from the facts that $||\sqrt{\tilde{\bp}_i}||_2=1$, $ ||\mathbf{a}^*_i||_2=1$, \eqref{eq:one-over-square-bound}, Inequality (5) comes from the fact that  $||\mathbf{p}_i(j)-\tilde{\mathbf{p}}_i(j)||_{\infty} \le \gamma {K^{n-1}}$, and Inequality (6) comes from the fact that $\frac{\gamma K^{n} }{ \delta^{3/2}} \le 1$ by the assumption of Theorem \ref{thm:continiousnmc}.

We claim $\bc_i$ is feasible for optimization \eqref{eq:temp2cont}. First note that the norm of each $\mathbf{c}_i$ is one. We next show that each $\mathbf{c}_i$ is orthogonal to $\sqrt{\tilde{\mathbf{p}}_i}$:
 \begin{align*}
 <\mathbf{c}_i, \sqrt{\tilde{\mathbf{p}}_i}> & = \frac{1}{||\mathbf{a}^*_i+ \sqrt{\tilde{\bp}_i} <\mathbf{a}^*_i, \sqrt{\mathbf{p}_i}- \sqrt{\tilde{\mathbf{p}}_i}>||_2} \\
 & \left( <\mathbf{a}^*_i , \sqrt{\tilde{\mathbf{p}}_i}>+ <\mathbf{a}^*_i , \sqrt{\mathbf{p}_i}- \sqrt{\tilde{\mathbf{p}}_i}> ||\sqrt{\tilde{\mathbf{p}}_i}||_2^2\right)= 0,
 \end{align*}
where the last equality follows from $||\sqrt{\tilde{\mathbf{p}}_i}||_2=1$ and $\mathbf{a}_i \perp \sqrt{\mathbf{p}_i}$. We now plug in the feasible solution $\mathbf{c}_i$ into the objective function of optimization \eqref{eq:temp2cont}.
\begin{align}\label{eq:bound-rho-tilde}
\tilde{\rho}_G  & \stackrel{(7)}{\ge} \sum_{(i, i')\in E} \mathbf{c}^T_i \tilde{Q}_{ii'} \mathbf{c}_{i'} \\
& \stackrel{(8)}{=} \sum_{(i, i')\in E} {\mathbf{a}^*}^T_i Q_{i,i'} \mathbf{a}^*_{i'} + \mathbf{c}^T_i \left( \tilde{Q}_{i,i'}- Q_{i,i'} \right) \mathbf{c}_{i'} + \left( \mathbf{c}^T_i- {\mathbf{a}^*}^T_i \right)  Q_{i,i'} \mathbf{c}_{i'}\nonumber\\
&+ {\mathbf{a}^*}^T_i Q_{i,i'} \left( \mathbf{c}_{i'}- \mathbf{a}^*_{i'}\right)  \nonumber\\
& \stackrel{(9)}{\ge} \rho_G - \sum_{(i, i')\in E} \Bigg( \underbrace{\left|\mathbf{c}^T_i \left( \tilde{Q}_{i,i'}- Q_{i,i'} \right) \mathbf{c}_{i'} \right|}_\text{Term I} +  \underbrace{\left| \left( \mathbf{c}^T_i- {\mathbf{a}^*}^T_i \right)  Q_{i,i'} \mathbf{c}_{i'} \right|}_\text{Term II}\nonumber\\
&+ \underbrace{\left| {\mathbf{a}^*}^T_i Q_{i,i'} \left( \mathbf{c}_{i'}- \mathbf{a}^*_{i'}\right) \right|}_\text{Term III}\Bigg),\nonumber
\end{align}
where Inequality (7) comes from the fact that $\bc_i$ for $1\leq i\leq n$ is a feasible solution of optimization \eqref{eq:temp2cont}, Equality (8) comes from adding and subtracting ${\mathbf{a}^*}^T_i Q_{i,i'} \mathbf{a}^*_{i'}$, $\mathbf{c}^T_i Q_{ii'} \mathbf{c}_{i'}$ and ${\mathbf{a}^*}^T_i Q_{i,i'} \mathbf{c}_{i'}$, and Inequality (9) comes from the triangle inequality and the fact that $\ba_i^*$,  $1\leq i\leq n$ is a solution of optimization \eqref{eq:temp1cont}. Next, we bound terms I-III on the righthand side of this relation.

Using Lemma \autoref{lem:diffQbasedondiffP}, for any $i, i'$, we have
\begin{align}\label{eq:bound-term1}
|\mathbf{c}^T_i \left( \tilde{Q}_{i,i'}- Q_{i,i'} \right) \mathbf{c}_{i'} | \le ||\tilde{Q}_{i,i'}- Q_{i,i'}||_2 \le 2 \frac{\gamma K^{n} }{\delta^2}.
\end{align}
We also have
\begin{align*}
& | \left( \mathbf{c}^T_i- {\mathbf{a}^*}^T_i \right)  Q_{i,i'} \mathbf{c}_{i'} |  \le  ||\mathbf{c}^T_i- {\mathbf{a}^*}^T_i ||_2 ||Q_{i,i'}||_2 ||\mathbf{c}_{i'}||_2 = ||\mathbf{c}^T_i- {\mathbf{a}^*}^T_i ||_2 \\
& \le 2 \frac{||\nu_i||_2}{1- ||\nu_i||_2},\nonumber
\end{align*}
where we use the following inequality
\begin{align}\label{eq:bound-temp11}
||\frac{\mathbf{a}+ \nu_i}{||\mathbf{a}+ \nu_i||_2}- \mathbf{a}||_2 &= ||\frac{\nu_i}{||\mathbf{a}+ \nu_i||_2} + \mathbf{a}\left(\frac{1}{||\mathbf{a}+ \nu_i||_2}-1  \right)||_2 \\
& \le ||\frac{\nu_i}{||\mathbf{a}+ \nu_i||_2} ||_2+ ||\mathbf{a} ||_2 |\frac{1}{||\mathbf{a}+ \nu_i||_2}-1  | \nonumber\\
& \le \frac{||\nu_i||_2}{1- ||\nu_i||_2} + \max\{|\frac{1}{1- ||\nu_i||_2}-1  |,  |\frac{1}{1+ ||\nu_i||_2}-1  |\}\nonumber\\
& \le 2 \frac{||\nu_i||_2}{1- ||\nu_i||_2},\nonumber
\end{align}
where we used $1- ||\nu_i||_2 \le ||\mathbf{a}+ \nu_i||_2 \le 1+ ||\nu_i||_2$ and the fact that $||\nu_i||\leq 1/2$ according to \eqref{eq:norm-nu-small}. Using \eqref{eq:norm-nu-small} and \eqref{eq:bound-temp11}, we obtain

\begin{align}\label{eq:bound-term2}
| \left( \mathbf{c}^T_i- {\mathbf{a}^*}^T_i \right)  Q_{i,i'} \mathbf{c}_{i'} | \le 2\frac{\gamma K^{n-1} }{ \delta^{3/2}} ,
\end{align}

Using \eqref{eq:bound-term1} and \eqref{eq:bound-term2} in \eqref{eq:bound-rho-tilde} leads to

\begin{align*}
\tilde{\rho}_G & \ge \rho_G -  \sum_{(i, i')\in E} \left( 4\frac{\gamma K^{n} }{ \delta^{2}}  + 4 \gamma K^{n-1} \frac{1}{\delta^{3/2}} \right) \ge  \rho_G -  \gamma K^{n} |E|  \frac{8}{\delta^2}.
\end{align*}
Similarly, we have
\begin{align*}
\rho_G & \ge \tilde{\rho}_G -  \sum_{(i, i')\in E} \left( 4\frac{\gamma K^{n} }{ \delta^{2}}  + 4 \gamma K^{n-1} \frac{1}{\delta^{3/2}} \right) \ge  \tilde{\rho}_G -  \gamma K^{n} |E|  \frac{8}{\delta^2}.
\end{align*}
Combining the previous two relations, we obtain
\begin{align*}
|\tilde{\rho}_G - \rho_G| \le \gamma K^{n} |E|  \frac{8}{\delta^2},
\end{align*}
which completes the proof.
\end{proof}
\subsection{Proof of Theorem \ref{thm:samplenmc}} \label{proof:thm:samplenmc}
\begin{proof}
We use the following theorem in the proof.
\begin{theorem}[\cite{devroye1983equivalence, berend2012convergence}]\label{Th:empiricalbound}
Let $P$ be a probability distribution on a finite alphabet $\mathcal{X}$. Also, let $P^{(m)}$ denote the empirical probability distribution of $X$, obtained from $m$ i.i.d. samples, $\{x_i\}_{i=1}^m$, drawn according to $P$. We have
\begin{align}
\mathbb{P}\left[||P^{(m)}- P||_{\infty} > \gamma \right] \le 4 e^{-m\frac{\gamma^2}{2}}.
\end{align}
\end{theorem}

Theorem \ref{Th:empiricalbound} establishes that $P$ and $P^{(m)}$ are close to each other for a sufficiently large sample size $m$\footnote{Note that concentration inequalities for empirical average of a discrete random variable has been also studied in several other works including \cite{abor1999sharp, weissman2003inequalities, massart2007concentration, boucheron2013concentration}. In particular, \cite{weissman2003inequalities} exploits properties of the underlying distribution to provide tight bounds on $\mathbb{P}[||P^{(m)}-P||_{1} \ge \gamma]$. This in turn provides a bound on $\mathbb{P}[||P^{(m)}-P||_{\infty} \ge \gamma]$ as well. However without additional assumptions on the underlying probability distribution this bound is of the same order as the one presented in Theorem \ref{Th:empiricalbound} (see \cite[Section 3]{weissman2003inequalities}).}. Moreover, one can use Theorem \ref{thm:continiousnmc} to bound the difference between NMC values of distributions $P$ and $P^{(m)}$ by $||P^{(m)}- P||_{\infty}$. However, to prove Theorem \ref{thm:samplenmc}, we also need to bound the minimum sample probability of $P^{(m)}$. We explain these steps below.

We have
\begin{align*}
& \mathbb{P}\left[ \delta_{X_i}(P^{(m)}) \ge \frac{\delta'}{2} \right]  \ge \mathbb{P}\left[ ||P_{X_i}^{(m)} - P_{X_i}||_{\infty} \le \frac{\delta'}{2} \right]\\
&\ge \mathbb{P} \left[ ||P^{(m)}- P||_{\infty} \le \frac{\delta'}{2 K^{n-1}} \right] \\
& \ge 1- \mathbb{P}\left[ ||P^{(m)}- P||_{\infty} > \frac{\delta'}{2 K^{n-1}} \right]  \ge 1- 4 e^{- m \frac{1}{2}(\frac{\delta'}{2 K^{n-1}})^2},
\end{align*}
where the first inequality follows from $\delta_{X_i}(P^{(m)}) \ge -|\delta_{X_i}(P^{(m)})- \delta_{X_i}(P)| + \delta_{X_i}(P) \ge \delta - ||P_{X_i}^{(m)}- P_{X_i}||_{\infty}$ and the second inequality follows from $||P_{X_i}^{(m)}- P_{X_i}||_{\infty} \le K^{n-1} ||P^{(m)}- P||_{\infty}$. Using the union bound, we obtain
\begin{align*}
 \mathbb{P}\left[ \bigcap_{i \in V} \left\{ \delta_{X_i}(P^{(m)})   \ge \frac{\delta'}{2} \right\} \right] & =1- \mathbb{P}\left[ \bigcup_{i \in V} \left\{ \delta_{X_i}(P^{(m)}) \le \frac{\delta'}{2} \right\} \right]  \\
& \ge 1- 4 n e^{- m \frac{1}{2}(\frac{\delta'}{2 K^{n-1}})^2}.
\end{align*}
By applying the union bound once more, we have
\begin{align*}
& \mathbb{P}\left[ \left\{||P^{(m)}- P||_{\infty} < \gamma\right\} \bigcap \left\{  \bigcap_{i \in V}  \left\{ \delta_{X_i}(P^{(m)}) \ge \frac{\delta'}{2}\right\} \right\}\right]  \\
& \ge 1- 4 n e^{- m \frac{1}{2}(\frac{\delta'}{2 K^{n-1}})^2} - 4 e^{-m \frac{\gamma^2}{2}}.
\end{align*}
Thus, with a large probability, i.e., $1- 4 n e^{- m \frac{1}{2}(\frac{\delta'}{2 K^{n-1}})^2} - 4 e^{-m \frac{\gamma^2}{2}}$, the minimum probability of $P^{(m)}$ is bounded by $\delta=\delta'/2$ and $||P^{(m)}-P||_{\infty}$ is bounded by $\gamma$.

For given $\eta<1$ and $\epsilon<1$, we let $\gamma= \frac{\epsilon (\delta')^2}{32 |E|} {K}^{-n}$ and
\begin{align*}
m_0 & \ge  2 \left(\frac{32 |E| K^n}{\epsilon (\delta')^2} \right)^2 \log \left( \frac{8 n }{\eta} \right).
\end{align*}

Now we use Theorem \ref{thm:continiousnmc} with $\delta=\delta'/2$. First note that $\gamma\leq \delta^{3/2}\cK^{-1}$ which is required as an assumption for invoking Theorem \ref{thm:continiousnmc}. This leads to
\begin{align*}
|\rho_G^{(m)}- \rho_G| \le \gamma |E| K^{n} \frac{8}{\delta^2} \le \epsilon,
\end{align*}
with probability (at least)
\begin{align*}
1- \left( 4 n e^{- m_0 \frac{1}{2}(\frac{\delta'}{2 K^{n-1}})^2} - 4 e^{-m_0 \frac{\gamma^2}{2}} \right) \ge 1- \eta.
\end{align*}
\end{proof}

\subsection{Proof of Theorem \ref{thm-nmc-guass-gen}}\label{proof:thm-nmc-guass-gen}
\begin{proof}

Let $\ba=(\ba_1,\cdots,\ba_n)$. Define
\begin{align}
U_L(\ba)\triangleq \sum_{(i,i')\in E} \ \sum_{j=1}^{L} a_{i,j} a_{i',j}\ (\rho_{i,i'})^j\\
U(\ba)\triangleq \sum_{(i,i')\in E} \ \sum_{j=1}^{\infty} a_{i,j} a_{i',j}\ (\rho_{i,i'})^j.\nonumber
\end{align}
Let $\tilde{\ba}^{(L)}=(\tilde{\ba}_1^{(L)}, \dots, \tilde{\ba}_n^{(L)})$ and $\tilde{\ba}=(\tilde{\ba}_1, \dots, \tilde{\ba}_n)$ be a solution of the following optimizations
\begin{align}\label{opt:uk-ck}
\max_{\ba} \quad & U_L(\ba)\\
& \ba \in \mathcal{C'}, \nonumber
\end{align}
and
\begin{align}
\max_{\ba} \quad & U(\ba)\\
& \ba \in \mathcal{C'}, \nonumber
\end{align}
respectively, where $\mathcal{C'}$ is the feasible set of the optimization \eqref{eq:nmc-gauss1}.  We first show for any $L$ and any $1 \le i\le n$, $\tilde{\ba}_i^{(L)} =(0, s_i , \dots, 0)$. We then show for any $ 1 \le i \le n$, $\tilde{\ba}_i$ must be equal to $(0, s_i, \dots)$, completing the proof.
\\\textbf{Claim 1:} We first characterize a global maximizer of optimization \eqref{opt:uk-ck}. Note that the constraint set $\mathcal{C'}$ can be truncated w.l.o.g. to consider $1\leq j\leq L$. Let $\Lambda$ be the matrix of correlation coefficients where $[\Lambda]_{i,i'}=\rho_{i,i'}$. Diagonal elements of $\Lambda$ are all zero, as we ignore self-loops. Define
\begin{align}
\bx= \left( a_{1,1}, a_{2,1}, \dots, a_{n,1}, a_{1,2}, \dots, a_{n,L} \right)^T.
\end{align}
\noindent
Moreover, define $A_0$ as an $nL \times nL$ matrix composed of $L^2$ blocks of size $n\times n$ whose $m$-th diagonal block is equal to $2\ \Lambda^{.m}$, where $\Lambda^{.m}[i,j]\triangleq (\Lambda[i,j])^m$. Off-diagonal blocks of $A_0$ are all zeros. Moreover, define $A_i$ for $1\leq i\leq n$ as an $nL \times nL$ matrix where $A_i[i+(m-1)n,i+(m-1)n]=1$ for $1\leq m\leq L$, otherwise it is zero. Therefore, optimization \eqref{opt:uk-ck} can be re-written as the following standard quadratic optimization:
\begin{align}\label{eq:nmc-gauss-standard-form}
\max_{\bx} \quad & \frac{1}{2} \bx^T A_0 \bx\\
& \frac{1}{2}\bx^T\ A_i\ \bx-\frac{1}{2}\leq 0, ~~ 1\leq i \leq n. \nonumber
\end{align}
Note that equality constraints of optimization \eqref{opt:uk-ck} are replaced by inequality ones in optimization \eqref{eq:nmc-gauss-standard-form}. This is because, since $(\rho_{i,i'})^2\geq 0$, solutions of optimization \eqref{eq:nmc-gauss-standard-form} belong to the boundary of its feasible set. Optimization \eqref{eq:nmc-gauss-standard-form} is a non-convex quadratic optimization with quadratic constraints. Reference \cite{jeyakumar2007non} (Proposition 3.2)\footnote{Optimization \eqref{eq:nmc-gauss-standard-form} can be stated as a minimization by replacing $-A_0$ instead of $A_0$ to use Proposition 3.2 of Reference \cite{jeyakumar2007non}} characterizes necessary and sufficient conditions for global minimizers of a generalized form of optimization \eqref{eq:nmc-gauss-standard-form}. Let
\begin{align}
\bar{\bx}=[s_1, s_2, \dots, s_n, 0, \dots, 0]^T,
\end{align}
where $s_i\in\{-1,1\}$, for $1\leq i\leq n$. Using Proposition 3.2 and Theorem 3.1 of \cite{jeyakumar2007non}, to have $\bar{\bx}$ as a global minimizer of optimization \eqref{eq:nmc-gauss-standard-form}, we need to have
\begin{align}\label{eq:cond1-optimality}
\big(\sum_{i=1}^{n}\lambda_i A_i-A_0\big)\bar{\bx}=0
\end{align}
and
\begin{align}\label{eq:cond2-optimality}
\sum_{i=1}^{n}\lambda_i A_i-A_0 \succeq 0,
\end{align}
where $\lambda_i\geq 0$, and $A\succeq 0$ means that $A$ is a positive semi-definite matrix. Using definitions of $A_0$, $A_i$, and $\bar{\bx}$, equation \eqref{eq:cond1-optimality} is satisfied iff
\begin{align}\label{eq:lambda-cond1}
\lambda_i=2\sum_{i'\neq i} s_i s_{i'} \rho_{i,i'} \geq 0, \quad 1\leq i\leq n.
\end{align}
Using \eqref{eq:lambda-cond1} and Gerschgorin's circle theorem, if
\begin{align}
& \sum_{i'\neq i} (1-s_i s_{i'})\rho_{i,i'} \geq 0,\quad \forall 1\leq i\leq n,\\
& \sum_{i'\neq i} s_i s_{i'} \rho_{i,i'} \geq \sum_{i'\neq i} \rho_{i,i'}^2,\quad \forall 1\leq i\leq n,
\end{align}
conditions \eqref{eq:cond2-optimality} are satisfied. Thus, $\bar{\bx}$ is a global minimizer of optimization \eqref{eq:nmc-gauss-standard-form}, establishing that for any $L$ and any $1 \le i \le n$ we have $\tilde{\ba}_i^{(L)}=(0, 1, \dots, 0)$.
\\\textbf{Claim 2:} Let $\ba^*=(\ba_1^*, \dots, \ba_n^*)$ where $\ba_i^*=(0, s_i, \dots)$ for all $ 1\le i \le n$. We next show that $\tilde{\ba}=\ba^*$. We proceed by contradiction. Let
\begin{align}
\Delta\triangleq U(\tilde{\ba})-U(\ba^{*}).
\end{align}
By the contradiction assumption $\Delta>0$. Since $(\rho_{i,i'})^L\to 0$ as $L\to\infty$ for all $i\neq i'$, $U_L(\cdot)$ converges uniformly to $U(\cdot)$. Thus there exists $L_0$ such that for $L\geq L_0$ and any $\ba$, we have
\begin{align}
|U_L(\ba)-U(\ba)|\leq \frac{\Delta}{4}.
\end{align}
Therefore we have
\begin{align}
U_L(\tilde{\ba})&\geq U(\tilde{\ba})-\frac{\Delta}{4}\\
& = U(\ba^{*})+\Delta-\frac{\Delta}{4}\nonumber\\
& \geq U_L(\ba^{*})-\frac{\Delta}{4}+\Delta-\frac{\Delta}{4}\nonumber\\
&= U_L(\ba^{*})+ \frac{\Delta}{2}> U_L(\ba^{*}).
\end{align}
This is in contradiction with the assumption that $\ba^{*}$ is the maximizer of $U_L(.)$. Putting Claims 1 and 2 together completes the proof.
\end{proof}
\subsection{Proof of Proposition \ref{prop:nmc-max-cut}}\label{proof:nmc-max-cut}
\begin{proof}
Under assumptions of Theorem \ref{thm-nmc-guass-gen} and using the definition of Hermite-Chebyshev polynomials \eqref{eq:hermitte-poly}, we can restrict the feasible set of optimization \eqref{eq:nmc} to the set of functions $\phi_i(X_i)=s_i X_i$ where $s_i\in \{-1,1\}$ for all $1\leq i\leq n$. Moreover, we have
\begin{align*}
\EE[\phi_i(X_i) \phi_{i'}(X_{i'})]=s_i s_{i'} \rho_{i,i'}.
\end{align*}
\noindent
Furthermore, $\EE[s_i X_i]=\EE[X_i]=0$, and $\EE\big[(s_i X_i)^2\big]=\EE[X_i^2]=1$. This completes the proof.
\end{proof}

\subsection{Proof of Theorem \ref{thm:nonlinear-guassian-graphical-model}}\label{proof:nonlinear-gaussian}
\begin{proof}
The proof of the first part of this Theorem is straightforward. To prove the second part, we re-write optimization \eqref{eq:nmc-Y} as follows:
\begin{align}\label{eq:nmc-Y2}
\max & \sum_{(i, i')} \mathbb{E}[g_i(f_i(X_i))\ g_{i'}(f_{i'}(X_{i'}))],\\
&\mathbb{E}[g_i(f_i(X_i))]=0, ~~ 1\leq i \leq n,\nonumber\\
&\mathbb{E}[g_i(f_i(X_i)^2]=1, ~~ 1\leq i \leq n.\nonumber
\end{align}
\noindent
Define $\phi_i(X_i)=g_i(f_i(X_i))$ for $1\leq i\leq n$. Since $f_i$'s are bijective, feasible regions of optimizations \eqref{eq:nmc-Y2} and \eqref{eq:nmc} are equal. Under the assumptions of Corollary \ref{cor:nmc-guass-all-ones}, $\phi_i^*(X_i)=X_i$. Thus, $g_i^*(f_i(X_i))=X_i$.
\end{proof}

\section{Discussion and Conclusion}\label{sec:disscussion}
The techniques we have developed in this paper can be used in other related formulations as well. In the following, we briefly highlight two examples of such formulations, namely absolute NMC and regularized NMC. Considering further properties of these optimizations is beyond the scope of the present paper.
\vspace{-0.3cm}
\subsection{Other Objective Functions}\label{Sec:SI-NMC}
The optimization \eqref{eq:nmc} maximizes aggregate pairwise correlations over the network. In some applications, the strength of an association does not depend on the sign of the correlation coefficient. In those cases, one can re-write the NMC optimization \eqref{eq:nmc} to maximize the total absolute pairwise correlations over the graph as follows:
\begin{definition}[Absolute Network Maximal Correlation]\label{def:anmc}
\textup{Consider the following optimization:
\begin{align}\label{eq:nmc-wit-abs}
\rho_{G}^{A}(X_1,\ldots,X_n)\triangleq \sup_{\phi_1, \dots, \phi_n}\ \sum_{(i, i')\in E} \left| \mathbb{E}[\phi_i(X_i)\ \phi_{i'}(X_{i'})] \right|,
\end{align}
such that $\phi_i: \mathcal{X}_i \to \mathbb{R}$ is Borel measurable, $\EE[\phi_i(X_i)]=0$, and $\EE[\phi_i(X_i)^2]=1$, for all $1\leq i\leq n$. $G=(V,E)$ is a graph with vertices $V=\{1,2,\ldots, n\}$ and edges $E=\{(i,i'):i,i'\in V, i\neq i'\}$. We refer to this optimization as an absolute NMC optimization.}
\end{definition}
\textup{A similar approach can be used to characterize the absolute NMC optimization \eqref{eq:nmc-wit-abs} by introducing extra variables $s_{i,i'}$ to represent correlation signs of edges $(i,i')$:}
\begin{align}\label{eq:nmc-opt-dis1_abs}
\sup_{\phi_1, \dots, \phi_n} \quad & \sum_{(i,i')\in E}  s_{i, i'} \EE[\phi_i(X_i) \phi_{i'}(X_{i'})] \nonumber \\
& \EE[\phi_i(X_i)]=0, ~ \EE[\phi_i^2(X_i)]=1,\quad ~ 1\leq i\leq n \nonumber\\
& s_{i,i'}\in \{-1,1\},\quad ~ 1\leq i,i' \leq n.
\end{align}
The NMC optimization \eqref{eq:nmc} results in $n$ possibly nonlinear transformation functions $\phi_i^*(X_i)$ whose correlation with the original variables can be small. In some applications, one may wish to restrict the set of possible transformations of the NMC optimization to control the correlation  between transformed and original variables. This can be done by introducing a regularization term in the definition of NMC as follows.
\begin{definition}[Regularized NMC]\label{def:regularized-nmc}
\textup{
The regularized NMC of real-valued $X_1, \dots, X_n$ connected by a graph $G=(V,E)$ is defined as the solution of the following optimization:
\begin{align}\label{eq:reg-nmc}
\rho_{G}^{R} \left( X_1,\ldots,X_n \right)=  \sup_{\phi_1, \dots, \phi_n} &(1-\lambda)\sum_{(i, i')\in E}\mathbb{E}\left[ \phi_i(X_i)\ \phi_{i'}(X_{i'}) \right]\\
&+ \lambda \sum_{i\in V} \mathbb{E} \left[ \phi_i(X_i)\ (X_i- \mathbb{E}[X_i] ) \right],\nonumber
\end{align}
such that $\phi_i: \mathcal{X}_i \to \mathbb{R}$ is Borel measurable, $\EE[\phi_i(X_i)]=0$, and $\EE[\phi_i(X_i)^2]=1$, for all $1\leq i\leq n$. $G=(V,E)$ is a graph with vertices $V=\{1,2,\ldots, n\}$ and edges $E=\{(i,i'):i,i'\in V, i\neq i'\}$. $0\leq \lambda\leq 1$ is the regularization parameter.
}
\end{definition}
Unlike MC and NMC, which only depend on the joint distributions of variables, the regularized NMC depends on both joint distributions and support of variables because of the regularization term. Moreover, one can define regularized absolute NMC similarly to the optimization \eqref{eq:nmc-wit-abs}.
\\ Let the optimal transformation functions computed by optimization \eqref{eq:reg-nmc} be $\phi_{i,\lambda}^*(X_i)$. If $\lambda=0$, $\phi_{i,\lambda}^*(X_i)=\phi_{i}^*(X_i)$, while if $\lambda=1$, $\phi_{i,\lambda}^*(X_i)=X_i$. By varying $\lambda$ between 0 and 1, $\phi_{i,\lambda}^*(X_i)$ vary from $\phi_i^*(X_i)$ to $X_i$. Define
\begin{equation}
R_{G,\lambda}(X_1,\ldots,X_n)\triangleq \sum_{(i,i')\in E} \mathbb{E} \left[ \phi_{i,\lambda}^*(X_i) \ \phi_{i',\lambda}^*(X_{i'}) \right].\nonumber
\end{equation}
Therefore, $R_{G,0}=\rho_{G}$ while $R_{G,1}$ is the total linear correlations over the network. By the definition of NMC, $R_{G,0}\geq R_{G,1}$. Note that we can use a similar algorithm to Algorithm \ref{alg:nmc} to compute regularized NMC of Definition \ref{def:regularized-nmc}. The objective function of the regularized NMC optimization \eqref{eq:reg-nmc} can be written as follows:
\begin{align}
\sum_{i\in V} \mathbb{E}\left[\phi_i(X_i)\ \left( \frac{1-\lambda}{2}\sum_{j\in \cN(i)} \phi_{i'}(X_{i'})+\lambda (X_i- \mathbb{E}[X_i]) \right)\right],
\end{align}
where $\cN(i)$ represents the set of neighbors of node $i$ in the graph $G=(V,E)$. To compute the regularized NMC, one can use an algorithm similar to the ACE Algorithm \ref{alg:nmc} with the following updates for transformation functions:
\begin{align}
\phi_i^{*}(X_i)=\EE \left[ \frac{1-\lambda}{2}\sum_{j\in \mathcal{N}(i)} \phi_{i'}(X_{i'})+\lambda (X_{i}- \mathbb{E}[X_i])|X_{i} \right].
\end{align}
\vspace{-0.7cm}
\subsection{Conclusion}\label{sec:conc}
In this paper, we propose NMC as a measure to capture nonlinear associations among variables. We show that NMC extends the standard bivariate MC to the case of having large number of variables, by assigning each variable to a single transformation function, thus avoiding over-fitting issues of using multiple MC optimizations over variable pairs. We also introduce a regularized NMC optimization which penalizes total distances of inferred transformed variables from the original ones. One can use other standard regularization techniques to further restrict inferred nonlinear functions in practical applications.

One of the main contributions of this work is providing a unifying framework to compute NMC (and therefore, MC) for both discrete and continuous variables using projections over appropriate Hilbert spaces. Using this framework, we establish a connection between the NMC optimization with the MCP and MEP for discrete random variables, and with the Max-Cut problem for jointly Gaussian variables. Using these relationships, we provide efficient algorithms to compute NMC in different cases. Note that properties of NMC for finite discrete variables characterized in Theorems \ref{thm:continiousnmc} and \ref{thm:samplenmc} depend on minimum marginal probabilities of variables (i.e., $\delta$). Therefore, extending these properties to the continuous case is not straightforward and requires exploiting techniques tailored for continuous variables and measures. To compute NMC for continuous random variables with general distributions, one can use the proposed optimization framework by choosing appropriate orthonormal basis for Hilbert spaces. For example, we use projections over Hermite-Chebyshev polynomials to characterize a solution of the NMC optimization for jointly Gaussian variables. Finally note that the inferred, possibly nonlinear, functions of the NMC optimization can be used in other applications such as nonlinear regression.
\vspace{-0.2cm}
\section{Acknowledgements}
The authors would like to thank Yue Li (MIT) and Gerald Quon (UC Davis) for helpful discussions regarding the cancer application, Yue Li for providing the cancer datasets, and Flavio P. Calmon (MIT) and Vahid Montazerhojjat (MIT) for helpful discussions. In addition, the authors would like to thank anonymous reviewers and associate editor for useful comments.
\vspace{-0.2cm}


\end{document}